\documentclass{article}

\PassOptionsToPackage{numbers}{natbib}

\usepackage[font=small]{caption}

\usepackage[final]{neurips_2022}

\usepackage{amsmath}

\usepackage{algorithm}
\usepackage{algpseudocode}
\usepackage{amsthm}

\newtheorem{proposition}{Proposition}[section]
\newtheorem{theorem}{Theorem}[section]
\newtheorem{assumption}{Assumption}[section]
\newtheorem{lemma}{Lemma}[section]
\usepackage{url}            
\usepackage{booktabs}       
\usepackage{amsfonts}       
\usepackage{nicefrac}       
\usepackage{microtype}      
\usepackage{enumitem}
\usepackage{color}
\usepackage[dvipsnames]{xcolor}
\usepackage{capt-of}
\usepackage{wrapfig}
\usepackage{multirow}
\usepackage{subcaption}
\usepackage{hyperref}
\hypersetup{colorlinks = true, 
            linkcolor = BrickRed,
            urlcolor  = Magenta,
            citecolor = RoyalBlue,
            anchorcolor = blue}

\usepackage[utf8]{inputenc} 
\usepackage[T1]{fontenc}    
\usepackage{hyperref}       
\usepackage{url}            
\usepackage{booktabs}       
\usepackage{amsfonts}       
\usepackage{nicefrac}       
\usepackage{microtype}      
\usepackage{xcolor}         
\usepackage{mathtools}
\usepackage[utf8]{inputenc} %
\usepackage[T1]{fontenc}    %
\usepackage{hyperref}       
\usepackage{url}            
\usepackage{booktabs}       
\usepackage{amsfonts}       
\usepackage{nicefrac}       
\usepackage{microtype}      
\usepackage{xcolor}         
\usepackage{mathtools}
\usepackage{nicefrac}       %
\usepackage{multirow}
\usepackage{graphicx}
\usepackage{makecell}
\usepackage{amsmath}
\usepackage{caption}

\usepackage{adjustbox}
\usepackage{array}
\usepackage{booktabs}
\usepackage{colortbl}
\usepackage{float,wrapfig}
\usepackage{hhline}
\usepackage{multirow}

\usepackage{amsmath,amsfonts,amsthm,amssymb}
\usepackage{bm}
\usepackage{nicefrac}
\usepackage{microtype}

\usepackage{changepage}
\usepackage{extramarks}
\usepackage{fancyhdr}
\usepackage{lastpage}
\usepackage{setspace}
\usepackage{soul}
\usepackage{xspace}

\usepackage{multirow}
\usepackage{array}
\usepackage{layouts}
\usepackage{algorithm}
\usepackage{bbding}
\usepackage{listings}
\usepackage{pifont}
\usepackage{wasysym}
\usepackage{float}
\usepackage{footnote}
\makesavenoteenv{tabular}
\makesavenoteenv{table}
\makesavenoteenv{figure}
\usepackage[]{xcolor}
\usepackage[font=small]{caption}
\usepackage{multirow}
\usepackage{amssymb}
\usepackage{pifont}

\usepackage{etoolbox}
\makeatletter
\AfterEndEnvironment{algorithm}{\let\@algcomment\relax}
\AtEndEnvironment{algorithm}{\kern2pt\hrule\relax\vskip3pt\@algcomment}
\let\@algcomment\relax
\newcommand\algcomment[1]{\def\@algcomment{\footnotesize#1}}
\renewcommand\fs@ruled{\def\@fs@cfont{\bfseries}\let\@fs@capt\floatc@ruled
  \def\@fs@pre{\hrule height.8pt depth0pt \kern2pt}%
  \def\@fs@post{}%
  \def\@fs@mid{\kern2pt\hrule\kern2pt}%
  \let\@fs@iftopcapt\iftrue}
\makeatother

\title{R\'enyiCL: Contrastive Representation Learning\\ with Skew R\'enyi Divergence}

\author{Kyungmin Lee \qquad Jinwoo Shin \\
Korea Advanced Institute of Science and Technology (KAIST) \\
\texttt{\{kyungmnlee, jinwoos\}@kaist.ac.kr}
}

\begin{document}

\maketitle

\begin{abstract}
    Contrastive representation learning seeks to acquire useful representations by estimating the shared information between multiple views of data.
    Here, the choice of data augmentation is sensitive to the quality of learned representations: as harder the data augmentations are applied, the views share more task-relevant information, but also task-irrelevant one that can hinder the generalization capability of representation.
    Motivated by this, we present a new robust contrastive learning scheme, coined R\'enyiCL, which can effectively manage harder augmentations by utilizing R\'enyi divergence.
    Our method is built upon the variational lower bound of R\'enyi divergence, but a na\"ive usage of a variational method is impractical due to the large variance.
    To tackle this challenge, we propose a novel contrastive objective that conducts variational estimation of a skew R\'enyi divergence and provide a theoretical guarantee on how variational estimation of skew divergence leads to stable training. 
    We show that R\'enyi contrastive learning objectives perform innate hard negative sampling and easy positive sampling simultaneously so that it can selectively learn useful features and ignore nuisance features.
    Through experiments on ImageNet, we show that R\'enyi contrastive learning with stronger augmentations outperforms other self-supervised methods without extra regularization or computational overhead. 
    Moreover, we also validate our method on other domains such as graph and tabular, showing empirical gain over other contrastive methods. 
    The implementation and pre-trained models are available at \footnote{{\url{https://github.com/kyungmnlee/RenyiCL}}}.
\end{abstract}

\section{Introduction}
Recently, many AI studies are enamored by the power of contrastive methods in learning useful representations without any human supervision, e.g., image~\cite{cmc,instdisc,simclr,mocov1,mocov2,mocov3}, text~\cite{text_1,text_2}, audio~\cite{audio_1, audio_2} and multimodal video~\cite{video_1, video_2, avidcma}.
The key components of contrastive learning can be summarized into two-fold: data augmentation and contrastive objectives.
The data augmentation generates different views of data where the views share information that is relevant to the downstream tasks. 
The contrastive objective then enforces the representation to capture the shared information between the views by bringing the views from the same data together and pushing the views from different data away in the representation space.
Therefore, the choice of data augmentation is critical in contrastive learning, cf., see~\cite{simclr,infomin,graphcl}.
If the views share insufficient information, the representations cannot learn sufficient features for downstream tasks.
On the other hand, if there is too much information between the views, the views might share nuisance features that hamper the generalization. 

In this paper, we propose a new contrastive learning objective using R\'enyi divergence~\cite{renyi}, coined \emph{R\'enyiCL}, which can effectively handle the case when the views share excessive information. 
The R\'enyi divergence is a generalization of KL divergence that is defined with an additional parameter called an order. As the order becomes higher, the R\'enyi divergence penalizes more when two distributions mismatch. Therefore, since the goal of contrastive representation learning is to train an encoder to discriminate between positives and negatives, we hypothesize that maximizing the R\'enyi divergence between positives and negatives could lead to more discriminative representations when the data augmentations are aggressive. 
Furthermore, our theoretical analysis shows that R\'enyiCL weighs importance on easy positives and hard negatives by using a higher order.

To implement R\'enyiCL, one can consider variational lower bound of R\'enyi divergence~\cite{birrell2021variational}. 
However, we show that a variational lower bound of R\'enyi divergence does not suffice for contrastive learning due to their large variance. 
To tackle this challenge, we first made a key observation that the existing contrastive objectives are a variational form of a skew KL divergence, and show that variational estimation of skew divergence induces low variance property. 
Inspired by this, we consider a variational lower bound of a skew R\'enyi divergence and use it to implement R\'enyiCL.

Through experiments under ImageNet~\cite{imagenet}, R\'enyiCL  achieves $76.2\%$ linear evaluation accuracy, outperforming other self-supervised learning methods even using significantly less training epochs. 
Moreover, we show that R\'enyiCL representations transfer to various datasets and tasks such as fine-grained object classification and few-shot classification, outperforming other self-supervised baselines. 
Finally, we validate the effectiveness of R\'enyiCL on various domains such as vision, graph, and tabular datasets by showing empirical gain over conventional contrastive learning methods. 

\section{Related Works}
Since Chen et al.~\cite{simclr} emphasized that the usage of data augmentation is crucial for contrastive learning, a series of works proposed various contrastive learning frameworks that showed great empirical success on visual representation learning~\cite{instdisc,simclr, mocov1,mocov2,mocov3,nnclr}.
Inspired by the empirical success, recent works focused on optimal view generation for both positive and negatives. 
Tian et al.~\cite{infomin} proposed the InfoMin principle for view generation, which states that the views should share minimal information that is relevant to the downstream tasks.
Another line of research focuses on hard negative samples~\cite{hcl, kalantidis2020hard, huynh2022boosting} to enhance the performance of contrastive learning. 
Especially, Robinson et al.~\cite{hcl} proposed to use importance sampling to learn with hard negative samples.

Meanwhile, others are interested in searching for different contrastive objectives by modeling with various probabilistic divergences between positives and negatives. 
Oord et al.~\cite{cpc} proposed Contrastive Predictive Coding (CPC) (as known as InfoNCE) that is widely used in various contrastive learning scheme such as \cite{instdisc, simclr,mocov1}.
They theoretically prove that CPC objective is a variational lower bound to the mutual information, later \cite{mlcpc} proposed Multi-Label CPC (MLCPC) which is a tighter lower bound to the mutual information.
Hjelm et al.~\cite{hjelm2018learning} proposed DeepInfoMax which used Jensen-Shannon divergence for variational estimation and maximization of mutual information. 
Zbontar et al.~\cite{zbontar2021barlow} proposed Barlow-Twins, which is shown to be a Hilbert-Schmidt Independence Criterion that approximates a Maximum Mean discrepancy measure between positives and negatives~\cite{tsai2021note}.
Ozair et al.~\cite{wpc} proposed Wasserstein Predictive Coding, which is both lower bound to the Wasserstein distance and mutual information. 
Similar to our approach, Tsai et al.~\cite{rpc} proposed relative predictive coding that uses skew $\chi^2$-divergence for contrastive learning, and empirically and theoretically show that variational estimation of skew $\chi^2$-divergence leads to stable contrastive representation learning.
Lastly, our work is similar to \cite{rince}, which proposed robust contrastive learning objective that uses symmetric binary classification loss for contrastive learning so that it can deal with noisy views.

\section{Preliminaries}
\subsection{R\'enyi divergence}
We first formally introduce R\'enyi divergence, which is a family of probability divergences including the popular Kullback-Leibler (KL) divergence as a special case. Let $P$, $Q$ be two probability distributions such that $P$ is absolutely continuous with respect to $Q$, denoted by $P\ll Q$ (i.e., the Radon-Nikodym derivative $dP/dQ$ exists). Then, the \emph{R\'enyi divergence}~\cite{renyi, van2014renyi} of order $\gamma\in(0,1)\cup(1,\infty)$ is defined by
\begin{align*}
    R_\gamma(P\,\|\,Q) \coloneqq 
    \frac{1}{\gamma(\gamma-1)}\log \mathbb{E}_{P}\bigg[\bigg(\frac{dP}{dQ}\bigg)^{\gamma-1}\bigg].
\end{align*}
It is intertwined with various $f$-divergences, e.g., 
R\'enyi divergence with $\gamma\leftarrow1$ and $\gamma=2$
becomes KL divergence $D_{\tt{KL}}(P\,\|\,Q)$ and a monotonic transformation of $\chi^2$ divergence, respectively.
Remark that the R\'enyi divergence with higher order $\gamma$, penalizes more when the probability mass of $P$ does not overlap with $Q$~\cite{minka2005divergence}. 
R\'enyi divergence has been studied in various machine learning tasks such as variational inference~\cite{li2016renyi,chen2018variational} and training~\cite{tao2018chi,yu2021pseudo} or evaluation of generative models~\cite{djolonga2020precision}.

\subsection{Variational lower bounds of mutual information}
The estimation and optimization of \emph{mutual information} is an important topic in various machine learning tasks including representation learning~\cite{var_bounds}.
The {mutual information} is defined by the KL divergence between the joint distribution and the product of marginal distributions. Formally, given two random variables $X\sim P_X$ and $Y\sim P_Y$, let $P_{XY}$ be the joint distribution of $X\times Y$. 
Then, we call the pair of samples $(x,y)\sim P_{XY}$ as a \emph{positive pair}, and $(x,y)\sim P_XP_Y$ as a \emph{negative pair}. 
Then, the mutual information is defined by KL divergence between positive pairs and negative pairs:
\begin{align*}
    \mathcal I(X;Y) \coloneqq D_{\tt{KL}}(P_{XY} \,\|\, P_XP_Y). 
\end{align*}
In general, the mutual information is intractable to compute unless the densities of $X$ and $Y$ are explicitly known, thus many works resort to optimize its variational lower bounds~\cite{var_bounds} associated with neural networks~\cite{mine}.
The Donsker-Varadhan~(DV)~\cite{dv} objective is a variational form of KL divergence defined as follows:
\begin{align}\label{dv}
    \mathcal{I}_{\tt{DV}}(f) \coloneqq \mathbb{E}_{P}[f] - \log\mathbb{E}_{Q}[e^{f}],\quad\text{where}\quad
    D_{\tt{KL}}(P\,\|\,Q) = \sup_{f\in\mathcal{F}}~\mathcal{I}_{\tt{DV}}(f), 
\end{align}
where $\mathcal{F}$ is a set of bounded measurable functions on the support of $P$ and $Q$, and the optimum $f^*\in\mathcal{F}$ satisfies $f^*\propto\log\frac{dP}{dQ}$. 
Then, Belghazi et al.~\cite{mine} proposed MINE that uses \eqref{dv} as follows:
\begin{align*}
    \mathcal{I}_{\tt{MINE}}(f) = \mathbb{E}_{P_{XY}}[f(x,y)] - \log \mathbb{E}_{P_XP_Y}[e^{f(x,y)}].
\end{align*}
However, Song et al.~\cite{song2019understanding} showed that estimation with MINE objective might occur large variance unless one uses large number of samples, and Tsai et al.~\cite{rpc} also empirically evidenced that contrastive learning with the DV objective suffers from training instability. 

To address the issue, the \emph{contrastive predictive coding}~(CPC) objective~(also as known as InfoNCE)~\cite{cpc} is a popular choice for various practices including contrastive representation learning~\cite{instdisc, mocov1, simclr}.
Given $B$ batch of samples $\{x_i\}_{i=1}^B$ from $X$, assume we have a single positive $y_i^+$ 
and $K$ negatives $\{y_{ij}^-\}_{j=1}^K$
for each $x_i$. Then, the CPC objective is defined as follows:
\begin{align*}
\mathcal{I}_{\tt{CPC}}(f) \coloneqq \mathbb{E}_{(x,y^+)\sim P_{XY},\{y_j^-\}_{j=1}^K\sim P_Y}\bigg[\log \frac{(K+1)\cdot e^{f(x,y^+)}}{e^{f(x,y^+)} + \sum_{j=1}^K e^{f(x,y_{j}^-)}}\bigg],
\end{align*}
where it is known that ${\mathcal{I}}_{\tt{CPC}}(f)\leq \mathcal I(X;Y)$ for any bounded measurable function $f$~\cite{cpc}.
However, as ${\mathcal{I}}_{\tt{CPC}}(f)\leq \log (K+1)$ for any $f$, the CPC objective becomes a high-bias estimator if the true value $\mathcal I(X;Y)$ is larger than $\log(K+1)$. 
To address this, Poole et al.~\cite{var_bounds} proposed $\alpha$-CPC which controls the bias by inserting $\alpha\in[0,1]$ into CPC as follows: 
\begin{equation}\label{eqn:cpc}
    \mathcal{I}_{\tt{CPC}}^{(\alpha)}(f) \coloneqq \mathbb{E}_{(x,y^+)\sim P_{XY},\{y_j^-\}_{j=1}^K\sim P_Y}\bigg[\log \frac{e^{f(x,y^+)}}{\alpha e^{f(x,y^+)} + \frac{1-\alpha}{K}\sum_{j=1}^K e^{f(x,y_{j}^-)}}\bigg],
\end{equation}
The $\alpha$-CPC objective also admits a variational lower bound to the mutual information $\mathcal I(X;Y)$ for any $\alpha\in[0,1]$~(recovers the original CPC when $\alpha = \frac{1}{K+1}$), and it achieves smaller bias when using smaller value of $\alpha$. 
Furthermore, Song et al.~\cite{mlcpc} proposed $\alpha$-\emph{Multi Label CPC}~(MLCPC) which provides even a tighter lower bound to the mutual information. While CPC can be considered as $(K+1)$-way $1$-shot classification problem, given $B$ batch of anchors $\{x_i\}_{i=1}^B$, MLCPC is equivalent to $B(K+1)$-way $B$-shot multi-label classification problem as follows:\footnote{Remark that we use slightly different form as in \cite{mlcpc} to scale $\alpha$ to be reside in $[0,1)$.}
\begin{equation}\label{eqn:mlcpc}
    {\mathcal{I}}_{\tt{MLCPC}}^{(\alpha)}(f) = \mathbb{E}_{\substack{\{(x_i,y_i^+)\}_{i=1}^B\sim P_{XY}\\ \{y_{ij}^-\}_{j=1}^K\sim P_Y}}\bigg[\frac{1}{B}\sum_{i=1}^B \log \frac{ e^{f(x_i,y_i^+)}}{\frac{\alpha}{B} \sum_{i=1}^B e^{f(x_i,y_i^+)} + \frac{1-\alpha}{BK}\sum_{i=1}^B\sum_{j=1}^K e^{f(x_i,y_{ij}^-)}}\bigg].
\end{equation}

\section{Variational Estimation of Skew R\'enyi Divergence}
Similar to variational estimators of $D_{\tt{KL}}(P_{XY}\|P_XP_Y)$ in the previous section, we consider a variational estimator of R\'enyi divergence $R_\gamma(P_{XY} \| P_XP_Y)$, in particular, considering
its application to contrastive representation learning.
To that end, we first introduce the following known lemma which states variational representation of R\'enyi divergence similar to the DV objective in \eqref{dv}.
\begin{lemma}[\cite{birrell2021variational}]\label{lemma:R\'enyidv}
For distributions $P$, $Q$ such that $P \ll Q$, let $\mathcal{F}$ be a set of bounded measurable functions.
Then R\'enyi divergence of order $\gamma\in(0,1)\cup(1,\infty)$ admits following variational form:
\begin{align}\label{R\'enyidv}
    R_\gamma(P\,\|\,Q) = \sup_{f\in\mathcal{F}}~\mathcal{I}_{\tt{Renyi}}^{(\gamma)}(f) \quad\text{for}\quad \mathcal{I}_{\tt{Renyi}}^{(\gamma)}(f)\coloneqq \frac{1}{\gamma-1}\log\mathbb{E}_{P}[e^{(\gamma-1)f}] 
    -\frac{1}{\gamma}\log\mathbb{E}_{Q}[e^{\gamma f}],
\end{align}
where the optimum $f^*\in\mathcal{F}$ satisfies $f^*\propto\log \frac{dP}{dQ}$. Also, as $\gamma\rightarrow1$, \eqref{R\'enyidv} becomes DV objective in \eqref{dv}.
\end{lemma}

\subsection{Challenges in variational R\'enyi divergence estimation}
To perform contrastive representation learning, one can use the variational estimator \eqref{R\'enyidv} for $R_\gamma(P_{XY}\|P_XP_Y)$, similarly as did for KL divergence. 
However, the contrastive learning with \eqref{R\'enyidv} is impractical due to the large variance.
In Appendix~\ref{appendix:gaussian_mi}, we show that even for a simple synthetic Gaussian dataset, \eqref{R\'enyidv} is not available due to exploding variance when estimating R\'enyi divergence between two highly-correlated Gaussian distributions.
This exploding-variance issue of \eqref{R\'enyidv} resembles that of the variational mutual information estimator well observed in the literature~\cite{rpc, var_bounds, song2019understanding}. 
In the following theorem, we provide an analogous result for the R\'enyi variational objective \eqref{R\'enyidv}. 
\begin{theorem}\label{thm:highvar}
Assume $P\ll Q$, and ${\tt{Var}}_Q[dP/dQ] <\infty$. 
Let $P_m$ and $Q_n$ be the empirical distributions of $m$ i.i.d samples from $P$ and $n$ i.i.d samples from $Q$, respectively. Define
\begin{align*}
    \hat{\mathcal{I}}_{\tt{Renyi}}^{(\gamma)}(f) \coloneqq \frac{1}{\gamma-1}\log\mathbb{E}_{P_m}[e^{(\gamma-1)f}] - \frac{1}{\gamma}\log\mathbb{E}_{Q_n}[e^{\gamma f}],
\end{align*}
and assume we have $f^*\propto\log(dP/dQ)$.
Then $\forall\gamma > 1,\forall m \in\mathbb{N}$, we have
\begin{align*}
    \lim_{n\rightarrow\infty} n\cdot {\tt{Var}}_{P,Q}[\hat{\mathcal{I}}_{\tt{Renyi}}^{(\gamma)}(f^*)] \geq \frac{e^{\gamma^2 D_{\tt{KL}}(P\,\|\,Q)}-\gamma^2}{e^{2\gamma(\gamma-1)R_\gamma(P\,\|\,Q)}}.
\end{align*}
\end{theorem}
The proof of Theorem~\ref{thm:highvar} is in Appendix~\ref{appendix:pfthm31}.
Theorem~\ref{thm:highvar} implies that even though one achieves the optimal function for variational estimation, the variance of \eqref{R\'enyidv} could explode exponentially with respect to the ground-truth mutual information. This result is coherent with \cite{mcallester2020formal}, which identified the problems in variational estimation of mutual information. 

On the other hand, we recall that CPC and MLCPC are empirically shown to be low variance estimators~\cite{var_bounds, mlcpc}. 
Then, the natural question arises: what makes CPC and MLCPC fundamentally different from the DV objective?
In the following section, we answer to this question by showing that the CPC and MLCPC objectives are variational lower bounds of a skew KL divergence, and provide a theoretical evidence that variational estimators of skew divergence can have low variance. This insight will be used later for designing a low-variance estimator for  the desired R\'enyi divergence.

\subsection{Contrastive learning objectives are variational skew-divergence estimators}\label{subsec:skew}

We first introduce the definition of \emph{$\alpha$-skew KL divergence}.
For distributions $P$, $Q$ with $P\ll Q$ and for any $\alpha\in[0,1]$, the $\alpha$-skew KL divergence between $P$ and $Q$ is defined by the KL divergence between $P$ and the mixture $\alpha P+(1-\alpha)Q$:
\begin{align*}
    D_{\tt{KL}}^{(\alpha)}(P\,\|\,Q) \coloneqq D_{\tt{KL}}(P \,\|\, \alpha P + (1-\alpha)Q).
\end{align*}
One can see that the DV objective~\eqref{dv} for $\alpha$-skew KL divergence can be written as following:
\begin{align}\label{eqn: skew-dv}
    D_{\tt{KL}}^{(\alpha)}(P\,\|\,Q) = \sup_{f\in\mathcal{F}}~\mathbb{E}_{P}[f] - \log\big(\alpha \mathbb{E}_{P}[e^{f}]+(1-\alpha)\mathbb{E}_{Q}[e^{f}]\big).
\end{align}
Then following theorem reveals that $\alpha$-CPC and $\alpha$-MLCPC are variational lower bounds of $\alpha$-skew divergence between $P_{XY}$ and $P_XP_Y$, $D^{(\alpha)}\big(P_{XY}\,\|\,P_XP_Y\big)$
\begin{theorem}\label{thm:4.2}
For any $\alpha\in (0,1/2)$, the $\alpha$-CPC and $\alpha$-MLCPC are variational lower bound of $\alpha$-skew KL divergence between $P_{XY}$ and $P_XP_Y$, i.e., the following holds:
\begin{align*}
    \sup_{f\in\mathcal{F}}~\mathcal{I}_{\tt{MLCPC}}^{(\alpha)}(f) = \sup_{f\in\mathcal{F}}~\mathcal{I}_{\tt{CPC}}^{(\alpha)}(f) = D_{\tt{KL}}^{(\alpha)}(P_{XY} \,\|\, P_XP_Y)
\end{align*}
\end{theorem}
Here we provide a simple proof sketch and the full proof is in Appendix~\ref{appdx:pfthm42}. 
From the Jensen's inequality, one can see that the variational form in \eqref{eqn: skew-dv} is a lower bound to the $\alpha$-MLCPC. Thus, we have $D_{\tt{KL}}^{(\alpha)}(P_{XY}\,\|\,P_XP_Y)\leq \sup_f \mathcal{I}_{\tt{MLCPC}}^{(\alpha)}(f)$. On the other hand, we also show that $\mathcal{I}_{\tt{MLCPC}}^{(\alpha)}(f)\leq D_{\tt{KL}}^{(\alpha)}(P_{XY}\,\|\,P_XP_Y)$ for all $f\in\mathcal{F}$. Therefore, we show that $\alpha$-MLCPC is a variational lower bound of $D_{\tt{KL}}^{(\alpha)}(P_{XY}\,\|\, P_XP_Y)$. 
Note that the same argument holds for $\alpha$-CPC. 
Remark that Theorem~\ref{thm:4.2} implies that $\alpha$-CPC and $\alpha$-MLCPC are strictly loose bounds of mutual information. In particular, from the convexity of KL divergence, the following holds for any $\alpha>0$:
\begin{align*}
     D_{\tt{KL}}^{(\alpha)}(P_{XY} \,\|\, P_XP_Y) \leq (1- \alpha)D_{\tt{KL}}(P_{XY} \,\|\, P_XP_Y) < D_{\tt{KL}}(P_{XY}\,\|\,P_XP_Y).
\end{align*}
Also, Theorem~\ref{thm:4.2} reveals that $\alpha$-CPC and $\alpha$-MLCPC can be written as following population form:
\begin{align*}
    &\mathcal{I}_{\tt{CPC}}^{(\alpha)}(f) = \mathbb{E}_{P_{XY}}[f(x,y)] - \mathbb{E}_{P_X}\big[\log\big( \alpha\mathbb{E}_{P_{Y|X}}[e^{f(x,y)}] + (1-\alpha)\mathbb{E}_{P_Y}[e^{f(x,y)}]\big)\big] \\
    &\mathcal{I}_{\tt{MLCPC}}^{(\alpha)}(f) = \mathbb{E}_{P_{XY}}[f(x,y)] - \log\big(\alpha \mathbb{E}_{P_{XY}}[e^{f(x,y)}] + (1-\alpha)\mathbb{E}_{P_XP_Y}[e^{f(x,y)}]\big).
\end{align*}


\paragraph{Variational lower bounds of skew-divergence have low variance}
Further, we formally show that the variational estimation of empirical skew divergence has low variance, explaining why CPC and MLCPC objectives become low-variance estimators of the mutual information. 
The following theorem shows that the variance of the variational estimator can be adjusted with the right choice of $\alpha$.
\begin{theorem}\label{thm}
Assume $P\ll Q$ and $\text{Var}_Q[dP/dQ] <\infty$. Let $P_m$ and $Q_n$ be empirical distributions of $m$ i.i.d samples from $P$ and $n$ i.i.d samples from $Q$. Then define 
\begin{align*}
    \hat{\mathcal{I}}_{\tt{KL}}^{(\alpha)}(f) = \mathbb{E}_{P_m}[f] - \log\big(\alpha\mathbb{E}_{P_m}[e^{f}] + (1-\alpha)\mathbb{E}_{Q_n}[e^{f}]\big),
\end{align*}
and assume that there is $\hat{f}\in\mathcal{F}$ that $|\mathcal{I}_{\tt{KL}}^{(\alpha)}(\hat{f}) - D_{\tt{KL}}^{(\alpha)}(P_m\,\|\,Q_n)| < \varepsilon_f$ for some $\varepsilon_f>0$.
Then for $\forall\alpha < 1/8$, the variance of estimator satisfies 
\begin{align*}
    \text{Var}_{P,Q}\big[\hat{\mathcal{I}}_{\tt{KL}}^{(\alpha)}(\hat{f})\big] \leq c_1\varepsilon_f + \frac{c_2(\alpha)}{\min\{n,m\}} + \frac{c_3\log^2(\alpha m)}{m} +\frac{c_4\log^2(c_5 n)}{\alpha^2n},
\end{align*}
for some constants $c_1,c_3,c_4,c_5 >0$ that are independent of $n$, $m$, $\alpha$, and $D_{\tt{KL}}(P\|Q)$, and $c_2(\alpha)$ satisfies $c_2(\alpha) = \min\{\frac{1}{\alpha}, \frac{\chi^2(P\|Q)}{1-\alpha} \}$, where  $\chi^2(P\|Q) = \mathbb{E}_Q[(dP/dQ)^2]$.
\end{theorem}
The proof is in Appendix~\ref{appdx:pfthm32}.
Since we have $\chi^2(P\|Q) \geq e^{D_{\tt{KL}}(P\|Q)}-1$, if $\alpha$ is too small, the bound in Theorem~\ref{thm} is loose as in Theorem~\ref{thm:highvar}. Therefore, Theorem~\ref{thm} implies that one should use sufficiently large $\alpha$ to achieve low variance. 
Theorem.~\ref{thm} demonstrates that if one can find a sufficiently close critic for empirical skew KL divergence, the variance of the estimator is asymptotically bounded unless we choose large enough $\alpha=\alpha_n=\omega(n^{-1/2})$.

\subsection{Variational estimation of skew R\'enyi divergence}\label{sec:rmlcpc}
Inspired by the theoretical guarantee that skew divergence can achieve bounded variance, we now present a variational estimator of skew R\'enyi divergence.
Remark that the analogous version of Theorem~\ref{thm} can be achieved for R\'enyi divergence (see Appendix~\ref{thmrenyi}).
For any $\alpha\in[0,1]$ and $\gamma\in(0,1)\cup (1,\infty)$, the $\alpha$-skew R\'enyi divergence of order $\gamma$ is defined as:
\begin{align*}
    R_\gamma^{(\alpha)}(P \,\|\, Q) \coloneqq R_\gamma(P \,\|\, \alpha P + (1-\alpha)Q).
\end{align*}
Then, we define $(\alpha,\gamma)$-\emph{R\'enyi Multi-Label CPC}~(RMLCPC) objective by using the variational lower bound~\eqref{R\'enyidv} for $\alpha$-skew R\'enyi divergence between $P_{XY}$ and $P_XP_Y$: 
\begin{align*}
    \mathcal{I}_{\tt{RMLCPC}}^{(\alpha,\gamma)}(f)\coloneqq \frac{1}{\gamma-1}\log\mathbb{E}_{P_{XY}}[e^{(\gamma-1)f(x,y)}] -\frac{1}{\gamma}\log(\alpha\mathbb{E}_{P_{XY}}[e^{\gamma f(x,y)}] + (1-\alpha)\mathbb{E}_{P_XP_Y}[e^{\gamma f(x,y)}]),
\end{align*}
which satisfies $\sup_{f} \mathcal{I}_{\text{RMLCPC}}^{(\alpha,\gamma)}(f) = R_\gamma^{(\alpha)}(P_{XY}\,\|\, P_XP_Y)$ for any $\alpha\in[0,1]$ and $\gamma \in (0,1)\cup(1,\infty)$.

Note that the RMLCPC objective can be used for contrastive learning with multiple positive views. 
For example, when using multi-crops data augmentation~\cite{swav}, instead of computing contrastive objective for each crop, we gather all positive pairs and negative pairs and compute only once for the final loss. 
The pseudo-code for the RMLCPC objective is in Appendix Algorithm~\ref{alg:code}. 
\section{R\'enyi Contrastive Representation Learning}
In this section, we present \emph{R\'enyi contrastive learning}~(R\'enyiCL) where we use the RMLCPC objective for contrastive representation learning. 
As we discussed earlier, the R\'enyi divergence penalizes more when two distributions differ. Thus, when using harder data augmentations in contrastive learning, one can expect that R\'enyiCL can learn more discriminative representation.

\subsection{R\'enyi contrastive representation learning}
We begin with backgrounds for contrastive representation learning.
Given a dataset $\mathcal{X}$, the goal of representation learning is to train an encoder $g:\mathcal{X}\rightarrow\mathbb{R}^d$ that is a useful feature of $\mathcal{X}$.
Especially, contrastive representation learning enforces $g$ to discriminate between positive pairs and negative pairs~\cite{instdisc}, where positive pairs are generated by applying data augmentation on the same data and negative pairs are augmented data from different source data. 
Formally, let $V,V^\prime$ be random variables for augmented views from dataset $\mathcal{X}$. 
Then denote $P_{VV^\prime}$ be the joint distribution of a positive pair and $P_VP_{V^\prime}$ be the distribution of negative pairs. 
For the sake of brevity, we denote $(z,z^+)\sim P_{VV^\prime}$ be a positive pair, and $(z,z^-)\sim P_VP_{V^\prime}$ be a negative pair.
Also, let $z=g(v), z^\prime=g(v^\prime)$ be features of the encoder and $Z=g(V), Z^\prime=g(V^\prime)$ be corresponding random variables of feature distributions. Then we define $(z,z^+)\sim P_{ZZ^\prime}$ be features of positive views, and $(z,z^-)\sim P_ZP_{Z^\prime}$ be features of negative views.

The InfoMax principle~\cite{hjelm2018learning} for representation learning aims to find $g$ that preserves the maximal mutual information between $V$ and $V^\prime$ by following:
\begin{align*}
    \sup_{g:\mathcal{X}\rightarrow\mathbb{R}^d}~D_{\tt{KL}}(P_{ZZ^\prime}\,\|\, P_{Z} P_{Z^\prime}) = \sup_{g:\mathcal{X}\rightarrow\mathbb{R}^d}~ \mathcal I\left(Z;Z^\prime\right) \leq  \mathcal I(V;V^\prime),
\end{align*}
i.e., finds a neural encoder $g$ which discriminates between positive and negative pairs as much as possible.
Here, $\alpha$-CPC or $\alpha$-MLCPC are used for plug-in variational estimators of mutual information, for example contrastive learning with $\alpha$-MLCPC satisfies following:
\begin{align*}
    \sup_{g:\mathcal{X}\rightarrow\mathbb{R}^d}~\sup_{f\in\mathcal{F}}~\mathcal{I}_{\tt{MLCPC}}^{(\alpha)}(f,g) = \sup_{g:\mathcal{X}\rightarrow\mathbb{R}^d}~ D_{\tt{KL}}^{(\alpha)}(P_{ZZ^\prime} \| P_ZP_{Z^\prime}), 
\end{align*}
since 
$\mathcal{I}_{\tt{MLCPC}}^{(\alpha)}(f,g)$ is a variational estimator of $D_{\tt{KL}}^{(\alpha)}(P_{ZZ^\prime} \|P_ZP_{Z^\prime})$.

Now we present R\'enyi contrastive representation learning, which considers the following optimization problem by using skew R\'enyi divergence:
\begin{align*}
    \sup_{g:\mathcal{X}\rightarrow\mathbb{R}^d}~\sup_{f\in\mathcal{F}}~\mathcal{I}_{\tt{RMLCPC}}^{(\alpha,\gamma)}(f,g) = \sup_{g:\mathcal{X}\rightarrow\mathbb{R}^d}~ R_\gamma^{(\alpha)}(P_{ZZ^\prime}\|P_ZP_{Z^\prime}),
\end{align*}
since $\mathcal{I}_{\tt{RMLCPC}}^{(\alpha,\gamma)}(f,g)$ is a variational estimator of $R_{\gamma}^{(\alpha)}(P_{ZZ^\prime} \|P_ZP_{Z^\prime})$.

\subsection{Gradient analysis for R\'enyiCL}
In this section, we provide an analysis on the effect of $\gamma$ in R\'enyiCL based on the gradient of contrastive objectives.
For simplicity, we consider the case when $\alpha=0$, and we provide general analysis for $\alpha>0$ in Appendix~\ref{appendix:generalderivation}. 
Let $f_\theta$ be a neural network parameterized by $\theta$. By using the reparametrization trick, the gradient of $\mathcal{I}_{\tt{MLCPC}}(f_\theta)$ becomes 
\begin{align*}
    \nabla_\theta \mathcal{I}_{\tt{MLCPC}}(f_\theta) &= \mathbb{E}_{z,z^+}[\nabla_\theta f_\theta(z,z^+)] - \frac{\mathbb{E}_{z,z^-}[e^{f_\theta(z,z^-)}\nabla_\theta f_\theta(z,z^-)] }{\mathbb{E}_{z,z^-}[e^{f_\theta(z,z^-)}]} \\
    &= \mathbb{E}_{z,z^+}[\nabla_\theta f_\theta(z,z^+)] - \mathbb{E}_{{\tt{sg}}(q_\theta(z,z^-))}[\nabla_\theta f_\theta(z,z^-)],
\end{align*}
where $q_\theta(z,z^-) \propto e^{f_\theta(z,z^-)}$ is a self-normalized importance weights~\cite{owen2013monte}, and $\tt{sg}$ is a stop-gradient operator. Thus, the MLCPC objective is equivalent to following in terms of gradient~\cite{guo2021tight}:
\begin{align}\label{mlcpc_gradient}
    \mathcal{I}_{\tt{MLCPC}}(f_\theta) = \mathbb{E}_{z,z^+}[f_\theta(z,z^+)] - \mathbb{E}_{{\tt{sg}}(q_\theta(z,z^-))}[f_\theta(z,z^-)].
\end{align}
This shows that the original contrastive objective performs innate hard negative sampling with importance weight $q_\theta$~\cite{hcl}, as the gradient of negative pairs becomes larger as the value of $f_\theta$ becomes larger. 
On the other hand, for RMLCPC objective, by letting $\alpha=0$, and taking gradient with respect to $\theta$, we have
\begin{align*}
    \nabla_\theta \mathcal{I}_{\tt{RMLCPC}}^{(\gamma)}(f_\theta) &= \frac{\mathbb{E}_{z,z^+}[e^{(\gamma-1)f_\theta(z,z^+)}\nabla_\theta f_\theta(z,z^+)] }{\mathbb{E}_{z,z^+}[e^{(\gamma-1)f_\theta(z,z^+)}]} - \frac{\mathbb{E}_{z,z^-}[e^{\gamma f_\theta(z,z^-)}\nabla_\theta f_\theta(z,z^-)] }{\mathbb{E}_{z,z^-}[e^{\gamma f_\theta(z,z^-)}]} \\
    &= \mathbb{E}_{{\tt{sg}}(q_\theta(z,z^+;\gamma-1))}[\nabla_\theta f_\theta(z,z^+)] - \mathbb{E}_{{\tt{sg}}(q_\theta(z,z^-;\gamma))}[\nabla_\theta f_\theta(z,z^-)],
\end{align*}
where ${\tt{sg}}(q_\theta(z,z^+;\gamma-1)) \propto e^{(\gamma-1)f_\theta(z,z^+)}$ and ${\tt{sg}}(q_\theta(z,z^-;\gamma)) \propto e^{\gamma f_\theta(z,z^-)}$ are self-normalizing importance weights for each positive and negative term.
Hence, the RMLCPC objective is equivalent to the following in terms of gradient:
\begin{align}\label{renyimlcpc_grad}
    \mathcal{I}_{\tt{RMLCPC}}^{(\gamma)}(f_\theta) = \mathbb{E}_{{\tt{sg}}(q_\theta(z,z^+;\gamma-1))}[f_\theta(z,z^+)] - \mathbb{E}_{{\tt{sg}}(
    q_\theta(z,z^-;\gamma))}[f_\theta(z,z^-)].
\end{align}
Remark that \eqref{mlcpc_gradient} and \eqref{renyimlcpc_grad} have common form that $f_\theta$ is maximized for positive pairs and minimized for negative pairs, which is similar to contrastive divergence~\cite{ebm}.
On the other hand, \eqref{mlcpc_gradient} and \eqref{renyimlcpc_grad} have two differences: 
\begin{itemize}
    \item \textbf{Hard negative sampling}: the gradient weighs more on harder negatives, i.e., $(z,z^-)$ with high value of $f_\theta(z,z^-)$, as $\gamma$ increases~\cite{hcl}.
    \item \textbf{Easy positive sampling}: the gradient weighs more on easier positives, i.e., $(z,z^+)$ with high value of $f_\theta(z,z^+)$ as $\gamma \in(1,\infty)$ increases.
\end{itemize}
Therefore, when there are positive views that share task-irrelevant information, the RMLCPC objective resists updating, and it regularizes the model to ignore task-irrelevant information. 
The hard negative sampling with importance weight was proposed in~\cite{hcl}, but our analysis shows that CPC and MLCPC conduct intrinsic hard negative sampling. 
Also, the R\'enyi contrastive learning can control the level of hard negative sampling by choosing the appropriate $\gamma$.
Lastly, remark that our analysis is given for $\alpha=0$, but in practice we use nonzero value of $\alpha$. 
In Appendix~\ref{appendix:alphaexp}, we show that it requires sufficiently small values of $\alpha$ to have the effect of easy positive sampling for harder data augmentations.

\section{Experiments}
\begin{table}[t]
\begin{minipage}[t]{.39\linewidth}
\centering
\small
\caption{Linear evaluation on the ImageNet validation set. 
We report pre-training epochs and Top-1 classification accuracies~(\%).}
\begin{tabular}[b]{@{}lc cc@{}}
    \toprule
    Method 			            & Epochs & Top-1  \\
    \midrule
    SimCLR~\cite{simclr}		& 800 	 & 70.4 \\
    Barlow Twins~\cite{zbontar2021barlow} & 800    & 73.2 \\
    BYOL~\cite{byol} 			& 800 	 & 74.3 \\
    MoCo v3~\cite{mocov3} 		& 800 	 & 74.6 \\
    SwAV~\cite{swav} 	& 800 	& {75.3} \\
    DINO~\cite{swav} 	& 800 	& {75.3} \\
    NNCLR~\cite{nnclr} 	& 1000  & 75.6 \\
    C-BYOL~\cite{lee2021compressive}  		    & 1000   & 75.6 \\
    \midrule
    \textbf{R\'enyiCL}	        & {\bf200} 	& {\bf75.3} \\  
    \textbf{R\'enyiCL}	        & {\bf300} 	& {\bf76.2} \\  
\bottomrule
\label{tab:imagenet_linear}
\end{tabular}
\end{minipage}\hspace{\fill}%
\begin{minipage}[t]{0.58\linewidth}
\centering
\small
\caption{Semi-supervised learning results on ImageNet.
We report Top-1 and Top-5 classification accuracies~(\%) by fine-tuning a pre-trained ResNet-50 with 1\% and 10\% ImageNet datasets.}
\begin{tabular}[b]{@{}l cc c cc@{}}
\toprule
 & \multicolumn{2}{c}{1\% ImageNet } && \multicolumn{2}{c}{10\% ImageNet}\\
\cmidrule{2-3}
\cmidrule{5-6}
Method & Top-1 & Top-5 && Top-1 & Top-5 \\
\midrule
Supervised~\cite{simclr}      & 25.4 	& 48.4 	&& 56.4 	& 80.4 	 \\
SimCLR~\cite{simclr} 		    & 48.3 	& 75.5 	&& 65.6 	& 87.8 	 \\
BYOL~\cite{byol} 		    & 53.2 	& 78.4 	&& 68.8 	& 89.0 	 \\
SwAV~\cite{swav} 		    & 53.9 	& 78.5 	&& {70.2} 	& {89.9} 	 \\
Barlow Twins~\cite{zbontar2021barlow}    & 55.0  & 79.2  && 69.7     & 89.3 \\
NNCLR~\cite{nnclr}  		    & {56.4} 	& {80.7} 	&& 69.8 	& 89.3   \\
C-BYOL~\cite{lee2021compressive}  & {\bf60.6} 	& {\bf83.4} 	&& 70.5 	& 90.0   \\
\midrule
\textbf{R\'enyiCL}		& {56.4} 	& 80.6 	&& {\bf71.2} 	& {\bf90.3}   \\
\bottomrule
\label{tab:imagenet_semisup}
\end{tabular}
\end{minipage}
\end{table}
\begin{table}[t]
\centering
\setlength\tabcolsep{4pt} 
\small
\caption{Transfer learning performance on object classification datasets. For Aircraft and Flowers, we report mean per-class accuracy~(\%), and for VOC2007, we report 11-point mAP. Otherwise, we report Top-1 classification accuracies~(\%).}
\begin{tabular}{@{}lcccccccccc@{}}
\toprule
Method 	 		 &	CIFAR10 & CIFAR100 & Food101   & Flowers  & Cars  & Aircraft & DTD   & SUN397 & VOC2007 \\
\midrule
Supervised~\cite{simclr} 			 &  93.6    & 78.3     & 72.3      & 94.7       & 66.7  & 61.0     & 74.9  & 61.9   & 82.8 	  \\
SimCLR~\cite{simclr}           &  90.6    & 71.6     & 68.4      & 91.2      & 50.3  & 50.3     & 74.5  & 58.8   & 81.4 	  \\
BYOL~\cite{byol}             &  91.3    & 78.4     & 75.3      & 96.1      & {67.8}  & 60.6     & 75.5  & 62.2   & 82.5 	  \\
NNCLR~\cite{nnclr}            &  93.7    & \textbf{79.0}     & 76.7      & {95.1}       & 67.1  & \textbf{64.1}     & 75.5  & 62.5   & 83.0 	  \\
\midrule
\textbf{R\'enyiCL}  & 	\textbf{94.4}    & \textbf{79.0}     & \textbf{78.0} 	   & \textbf{96.5}     & \bf 71.5  & 61.8     & \textbf{77.3}  & \textbf{66.1}   & \textbf{88.2} 	  \\
\bottomrule
\label{tab:transfer}
\end{tabular}
\vspace{-10pt}
\end{table}

\subsection{R\'enyiCL for visual representation learning on ImageNet}
{\bf Setup.}
For ImageNet~\cite{imagenet} experiments, we use ResNet-50~\cite{resnet} for encoder $g$.
We use MLP projection head with momentum encoder~\cite{mocov1}, which is updated by EMA, and we use predictor, following the practice of ~\cite{byol, mocov3}.
Then the critic $f$ is implemented by the cosine similarity between the output of momentum encoder and base encoder with predictor, divided by temperature $\tau=0.5$.
We maximize RMLCPC objective with $\gamma=2.0$.

\vspace{0.05in}
{\bf Data augmentation.}
Given the popular data augmentation baseline~\cite{byol, simclr}, we further use RandAugment~\cite{ra} and multi-crop augmentation~\cite{swav} to implement harder data augmentation. 
When using multi-crop, we gather all positives and negatives, then directly compute the RMLCPC objective, which differs from the original method in~\cite{swav}~(see \ref{alg:code} in Appendix~\ref{appendix:generalderivation}).

\vspace{0.05in}
{\bf Linear evaluation.} 
We follow the linear evaluation protocol, where we report the Top-1 ImageNet validation accuracy~(\%) of a linear classifier trained on the top of frozen features.
Table~\ref{tab:imagenet_linear} compares the performance of different self-supervised methods using linear evaluation protocol. 
First, R\'enyiCL outperforms other self-supervised methods with a large margin by only training for 300 epochs.
Meanwhile, R\'enyiCL enjoys better computational efficiency in that it does not require a large batch size~\cite{simclr,byol,swav, mocov3} or extra memory queue~\cite{nnclr}.  

\vspace{0.05in}
{\bf Semi-supervised learning.}
We evaluate the usefulness of the learned feature in a semi-supervised setting with 1\% and 10\% subsets of ImageNet dataset~\cite{simclr, swav}.
The results are in Table~\ref{tab:imagenet_semisup}.
We observe that R\'enyiCL achieves the best performance on fine-tuning with 10\% subset, and runner-up on fine-tuning with 1\% subset, showing the generalization capacity in low-shot learning scenarios.

\vspace{0.05in}
{\bf Transfer learning.}
We evaluate the learned representation through the performance of transfer learning on fine-grained object classification datasets by using linear evaluation. 
We use CIFAR10/100~\cite{cifar}, Food101~\cite{food}, Flowers~\cite{flowers}, Cars~\cite{cars}, Aircraft~\cite{aircraft}, DTD~\cite{dtd}, SUN397~\cite{sun397}, and VOC2007~\cite{voc2007}. 
In Table~\ref{tab:transfer}, we compare it with other self-supervised methods. Note that R\'enyiCL achieves the best performance in 8 out of 9 datasets, especially showing superior performance on SUN397~(3.6\%) and VOC2007~(5.2\%) datasets. 

\begin{table}[t]
\small
\setlength\tabcolsep{2.5pt}
\caption{Comparison on contrastive self-supervised methods with harder augmentations. We compare Top-1 linear evaluation accuracy~(\%) on ImageNet, Pets, Caltech101 datasets and few-shot classification accuracy~(\%) over 2000 episodes on FC100, CUB200, and Plant disease datasets. $^\dag$ denotes trained without multi-crops. Here, we consider R\'enyiCL without multi-crops as other baselines do, i.e., for fair comparison.}
\centering
\begin{tabular}[t]{l c ccc c ccc c ccc}
    \toprule
          & & \multicolumn{3}{c}{Linear evaluation} && \multicolumn{3}{c}{5-way 1-shot} && \multicolumn{3}{c}{5-way 5-shot} \\
          \cmidrule{3-5}\cmidrule{7-9}\cmidrule{11-13}
    Method  & Epochs    & ImageNet    & Pets & Caltech101  && FC100 & CUB200 & Plant && FC100 & CUB200 & Plant \\
    \midrule
    InfoMin~\cite{infomin}  & 800 
    & \textbf{73.0}  & 86.37 & 88.48 &
    & 31.80 & 53.81 & 66.11 &
    & 45.09 & 72.20 & 84.12   \\
    CLSA~\cite{clsa}	& 800 
    & 72.2  & 85.18 & 91.21   &
    & 34.46 & 50.83 & 67.39 &
    & 49.16 & 67.93 & 86.15   \\
    \midrule
    \textbf{R\'enyiCL}$^\dag$  & 200 
    & 72.6  & \textbf{88.43} & \textbf{94.04}    &
    & \textbf{36.31} & \textbf{59.73} & \textbf{80.68} &
    & \textbf{53.39} & \textbf{82.12} & \textbf{94.29}    \\
    \bottomrule
\end{tabular}
\label{tab:hardcomp}
\vspace{-10pt}
\end{table}
\begin{table}[t]
\small
\setlength\tabcolsep{2.5pt} 
\caption{Ablation on contrastive objectives for unsupervised representation learning on image (ImageNet-100, CIFAR-100, CIFAR-10) and tabular (CovType, Higgs-100K) datasets. We report Top-1 accuracy~(\%) with linear evaluation. Average over 5 runs.}
\centering
\begin{tabular}[t]{l cc c cc c cc c|c cc c cc }
\toprule
        & \multicolumn{2}{c}{ImageNet-100} &&
                \multicolumn{2}{c}{CIFAR-100} && 
                \multicolumn{2}{c}{CIFAR-10} && &
                \multicolumn{2}{c}{CovType} && 
                \multicolumn{2}{c}{Higgs-100K} \\
                \cmidrule{2-3} \cmidrule{5-6} \cmidrule{8-9}
                \cmidrule{12-13} \cmidrule{15-16} 
    Method      & Base & Hard &
                & Base & Hard &
                & Base & Hard & & 
                & Base & Hard &
                & Base & Hard \\
    \midrule
    CPC         & 79.7 & 81.1(+1.4) &
                & 65.4 & 67.1(+1.7) &
                & 91.7 & 91.9(+0.2) & & 
                & 71.6 & 74.3(+2.7) &
                & 64.7 & 71.3(+6.6) \\
    MLCPC       & 79.5 & 81.2(+1.7) &
                & 65.6 & 66.6(+1.0) &
                & 91.9 & 92.1(+0.2) & & 
                & 71.7 & 74.1(+2.4) &
                & 64.9 & 71.5(+6.6) \\
    RMLCPC      & 78.9 & \textbf{81.6(+2.7)} &
                & 64.5 & \textbf{68.5(+4.0)} &
                & 90.7 & \textbf{92.5(+1.8)} & &
                & 72.1 & \textbf{74.9(+2.8)} &
                & 64.5 & \textbf{72.4(+7.9)}\\
    \bottomrule
\end{tabular}
\label{tab:ablation}
\end{table}
\vspace{0.05in}
{\bf Comparison among self-supervised methods using hard augmentations.}
To demonstrate the effectiveness of R\'enyiCL on learning with harder augmentations, we compare with InfoMin~\cite{infomin} and CLSA~\cite{clsa} where they also use harder augmentations for contrastive representation learning.
Table~\ref{tab:hardcomp} demonstrates the performance of representations learned with different self-supervised methods by linear evaulation on ImageNet~\cite{imagenet}, Pets~\cite{pets}, Caltech101~\cite{caltech}, and few-shot classification on FC100~\cite{fc100}, Caltech-UCSD Birds~(CUB200)~\cite{ra}, and Plant Disease~\cite{plant} datasets.
While R\'enyiCL achieves slightly lower performance than InfoMin on ImageNet validation accuracy, it outperforms InfoMin and CLSA on other transfer learning tasks with a large margin, which indicates the superiority of R\'enyiCL on learning better representations under the hard data augmentations.

\subsection{R\'enyiCL for representation learning on various domains}

\vspace{0.05in}
{\bf Setup.}
We perform ablation studies of R\'enyiCL on various domains such as images, tabular, and graph datasets.
For images, we consider CIFAR-10/100~\cite{cifar} and ImageNet-100~\cite{align_unif}, which is a 100-class subset of ImageNet~\cite{imagenet}.
For the base data augmentation, we use a popular benchmark from~\cite{simclr}, and for the harder augmentation, we further apply RandAugment~\cite{ra} and RandomErasing~\cite{re}.
For ImageNet-100, we use the default settings of~\cite{mocov3} with ResNet-50 backbone with only change in learning objective.
For CIFAR-10 and CIFAR-100, we follow the settings in~\cite{simclr} with ResNet-18 backbone adjusted for the CIFAR dataset. 
We use same $\alpha$ for all CPC, MLCPC, RMLCPC, and use $\gamma=2.0$ for RMLCPC.

For tabular experiments, we use Forest Cover Type~(CovType) and Higgs Boson~(Higgs)~\cite{higgs} dataset from UCI repository~\cite{Dua:2019}.
Due to its massive size, we consider a subset of 100K for the Higgs experiments.
Since the tabular dataset has limited domain knowledge, we do not use any data augmentation for the baseline and when using data augmentation, we use random masking noise~\cite{imix} for the Higgs dataset and random feature corruption~\cite{scarf} for the CovType dataset. 
We follow the settings of MoCo v3~\cite{mocov3} except that a 5-layer MLP is used for the backbone.
We use $\gamma=1.1$ for CovType and $\gamma=1.2$ for Higgs.
For evaluation, we use linear evaluation for CovType and we compute the linear regression layer by pseudo-inverse between the feature matrix and label matrix. 

Lastly, we examine R\'enyiCL on graph TUDataset~\cite{tudataset}, which contains numerous graph machine learning benchmarks of different domains such as bioinformatics, molecules, and social network. 
We follow the experimental setup of \cite{graphcl}: we use graph isomorphism network~\cite{gin} backbone and use node dropout, edge perturbation, attribute masking, and subgraph sampling for data augmentations.
For the graph experiments, we do not change the strength of data augmentation nor add new data augmentation.
The detailed hyper-parameters for training and evaluation for the whole dataset can be found in Appendix~\ref{appdx:dataset}.

\begin{table}[t]
\small
\setlength\tabcolsep{4.5pt} 
\caption{{Comparison of different contrastive learning methods on graph TUDataset. Average over 5 runs.}}
\centering
\begin{tabular}{l|cccc|cccc}
    \toprule
    Method                     & NCI1 & PROTEINS & DD & MUTAG &  COLLAB & RDT-B & RDT-M5K & IMDB-B\\
    \midrule
    InfoGraph~\cite{Infograph}  & 76.20 & 74.44 & 72.85 & 89.01 & 70.65 & 82.50 & 53.46 & \textbf{73.03} \\
    GraphCL~\cite{graphcl}      & 77.87 & 74.39 & 78.62 & 86.80 & 71.36 & 89.53 & 55.99 & 71.14 \\
    JOAO~\cite{joao}            & 78.07 & 74.55 & 77.32 & 87.35 & 69.50 & 85.29 & 55.74 & 70.21 \\
    JOAOv2~\cite{joao}          & 78.36 & 74.07 & 77.40 & 87.67 & 69.33 & 86.42 & 56.03 & 70.83 \\ 
    \midrule
    \textbf{R\'enyiCL}   & \textbf{78.60} & \textbf{75.11} & \textbf{78.98} & \textbf{90.22} & \textbf{71.88}          & \textbf{90.92}          & \textbf{56.18} & {72.38} \\
    \bottomrule
\end{tabular}
\label{tab:graph}
\vspace{-10pt}
\end{table}
\vspace{0.05in}
{\bf Results.}
Table~\ref{tab:ablation} compares the Top-1 linear evaluation accuracy~(\%) of representations with different contrastive learning objectives. 
We observe that RMLCPC obtains the best performance when using harder data augmentations on all images and tabular datasets.
Also, RMLCPC has the largest gain on using harder augmentation, showing that the effectiveness of R\'enyiCL is amplified when we use stronger augmentation. 
Table~\ref{tab:graph} compare the Top-1 linear evaluation accuracy~(\%) of graph R\'enyiCL with other graph contrastive learning methods.
Even without harder data augmentation, R\'enyiCL achieves the best performance in 7 out of 8 benchmarks.
Especially, R\'enyiCL outperforms JOAO~\cite{joao}, where the views are learned to boost the performance of graph contrastive learning. 

\subsection{Mutual information estimation and view selection}~\label{sec:expmi}
Remark that CPC and MLCPC are intrinsically high-bias mutual information estimator as $\alpha$-skew KL divergence is strictly smaller than original KL divergence for $0<\alpha<1$, i.e.,
\begin{align*}
    D_{\tt{KL}}^{(\alpha)}(P\|Q) \leq (1-\alpha)D_{\tt{KL}}(P\|Q) < D_{\tt{KL}}(P\|Q),
\end{align*}
from the convexity of KL divergence. 
Similarly, RMLCPC is also a strictly-biased estimator of R\'enyi divergence. 
However, one can recover the original mutual information by using the optimal condition of function $f$ from all CPC, MLCPC, and RMLCPC objectives~\cite{rpc}.
Here, we describe the sketch and present detailed method in Appendix~\ref{appendix:midetail}. 
From the optimality condition of \eqref{dv} and \eqref{R\'enyidv}, the optimal critic $f^*$ of $\alpha$-CPC, $\alpha$-MLCPC and $(\alpha,\gamma)$-RMLCPC of any $\gamma$ satisfies $f^* \propto \log \frac{dP_{XY}}{\alpha dP_{XY}+(1-\alpha)dP_XdP_Y}$, then we use Monte-Carlo method~\cite{owen2013monte} to approximate log-normalization constant, and use it to approximate the log-density ratio $\hat{r}=\log \frac{dP_{XY}}{dP_XdP_Y}$. Then we estimate the mutual information by $\hat{\mathcal I}(X;Y) = \frac{1}{B}\sum_{i=1}^B \log \hat{r}(x_i,y_i)$, where $(x_i,y_i)\sim P_{XY}$ are positive pairs.

\paragraph{Mutual information estimation of synthetic Gaussian.}
Following \cite{var_bounds, mlcpc}, we conduct experiments on estimating the mutual information of multivariate Gaussian distributions.  
In Appendix~\ref{appendix:gaussian_mi}, we show that while $\alpha$-CPC, $\alpha$-MLCPC, and $(\alpha,\gamma)$-RMLCPC objectives are high biased objectives, one can achieve low bias, low variance estimator by using the approximation method above. 

\section{Conclusion and Discussion}
We propose R\'enyi Contrastive Learning~(R\'enyiCL) which utilizes R\'enyi divergence to deal with stronger data augmentations.
Since the variational lower bound of R\'enyi divergence is insufficient for contrastive learning due to large variance, we introduce a variational lower bound of skew R\'enyi divergence, namely R\'enyi-MLCPC. 
Indeed, we show that CPC and MLCPC are variational forms of skew KL divergence, and provide theoretical analysis on how they achieve low variance. 
Through experiments, we validate the effectiveness of R\'enyiCL by using harder data augmentations. 

\textbf{Limitations and future works. }
While R\'enyiCL is beneficial when using harder data augmentations, we did not identified the optimal data augmentation strategy. 
One can use policy search method on data augmentation~\cite{cubuk2018autoaugment} with our R\'enyiCL similar to that of \cite{Reed_2021_CVPR}.
Also, we think that the proposed scheme based on R\'enyi divergence could be useful for other tasks, e.g., information bottleneck \cite{tishby2015deep}. Lastly, many works focused on the diverse perspectives of contrastive learning~\cite{guo2021tight,chen2021simpler, haochen2021provable}, where those approaches could be complementary to our work. We leave them for future works. 

\textbf{Negative societal impacts. }
Since contrastive learning often requires long epochs of training, it raises environmental concerns, e.g. carbon generation. Nevertheless, R\'enyiCL is shown to be effective even under a smaller number of epochs, compared to existing schemes.

\section*{Acknowledgement}
This work was partly supported by Institute of Information \& communications Technology Planning \& Evaluation (IITP) grant funded by the Korea government(MSIT)  
(No.2019-0-00075, Artificial Intelligence Graduate School Program (KAIST)) and Institute of Information \& communications Technology Planning \& Evaluation (IITP) grant funded by the Korea government(MSIT) (No.2022-0-00959, Few-shot Learning of Causal Inference in Vision and Language for Decision Making).



\bibliography{main}
\bibliographystyle{unsrtnat}

\newpage
\appendix
\section{Proofs}
\subsection{Proof of Theorem~\ref{thm:highvar}}\label{appendix:pfthm31}
\proof{
Without loss of generality, assume $e^{f^*} = dP/dQ$. Then we have
\begin{align}
    {\tt{Var}}_Q[e^{f^*}] &= \mathbb{E}_Q\left[\left(\frac{dP}{dQ}\right)^2\right] - \left(\mathbb{E}_Q\left[\frac{dP}{dQ}\right]\right)^2 = \mathbb{E}_P\left[\frac{dP}{dQ}\right] - 1,
\end{align}
since $\mathbb{E}_Q[dP/dQ] = 1$.
Then the variance of $n$ i.i.d random variable gives us
\begin{align*}
    {\tt{Var}}_Q\left[\mathbb{E}_{Q_n}[e^{f^*}]\right] = \frac{{\tt{Var}}[e^{f^*}]}{n} \rightarrow 0 \text{ as } n\rightarrow\infty
\end{align*}
For a sequence of random variable $X_n$ defined with respect to distribution $Q$, and assume we have $\lim_{n\rightarrow\infty}X_n =\mathbb{E}[X]$, then the following comes from the delta method:
\begin{align}\label{proof:1}
    \lim_{n\rightarrow\infty}n\cdot{\tt{Var}}_Q[f(X_n)] = (f'(\mathbb{E}[X]))^2\cdot{\tt{Var}}_Q[X].
\end{align}
Thus, by applying $f(t)=t^\gamma$ and $\mathbb{E}_Q[dP/dQ]=1$ gives us
\begin{align}\label{proof:2}
    \begin{split}
    \lim_{n\rightarrow\infty} n\cdot{\tt{Var}}_Q[\mathbb{E}_{Q_n}[e^{\gamma f^*}]] &= \gamma^2 \cdot {\tt{Var}}_Q[e^{f^*}] \\
    & = \gamma^2 (\mathbb{E}_P[dP/dQ] -1) \\
    &\geq e^{\gamma^2 \mathbb{E}_P[\log(dP/dQ)]} -\gamma^2 \\
    &= e^{\gamma^2 D_{\text{KL}}(P\|Q)} -\gamma^2,
    \end{split}
\end{align}
where the last two equations are from Jensen's inequality and the definition of KL divergence. 
Now, by applying $f(x)=\log x$ in \eqref{proof:1} , we have
\begin{align*}
    \lim_{n\rightarrow\infty} n\cdot{\tt{Var}}_Q[\log\mathbb{E}_{Q_n}[e^{\gamma f^*}]] &=\lim_{n\rightarrow\infty} \frac{n\cdot{\tt{Var}}_{Q}[\mathbb{E}_{Q_n}[e^{\gamma f^*}]]}{(\mathbb{E}_{Q_n}[e^{\gamma f^*}])^2}\\ &=\lim_{n\rightarrow\infty}\frac{n\cdot{\tt{Var}}_{Q}[\mathbb{E}_{Q_n}[e^{\gamma f^*}]]}{e^{2\gamma(\gamma-1)R_\gamma(P\|Q)}} \\
    &\geq \frac{e^{\gamma^2 D_{\text{KL}}(P\|Q)} - \gamma^2}{e^{2\gamma(\gamma-1)R_\gamma(P\|Q)}}
\end{align*}
from the definition of R\'enyi divergence and \eqref{proof:2}.
Thus, we have the following:
\begin{align*}
    \lim_{n\rightarrow\infty} n\cdot{\tt{Var}}[\mathcal{I}_{\tt{Renyi}}^{m,n}[e^{f^*}]] \geq \lim_{n\rightarrow\infty} n\cdot{\tt{Var}}[\log\mathbb{E}_{Q_n}[e^{f^*}]] \geq \frac{e^{\gamma^2 D_{\text{KL}}(P\|Q)} - \gamma^2}{e^{2\gamma(\gamma-1)R_\gamma(P\|Q)}}.
\end{align*}
}
\subsection{Proof of Theorem~\ref{thm:4.2}}\label{appdx:pfthm42}
To prove Theorem~\ref{thm:4.2}, we first state following Lemma.
\begin{lemma}{(Proposition 1 \& 2 in \cite{mlcpc})}\label{lemma:1}
Given positive integers $n\geq 1$ and $m\geq 2$, and for any collection of positive random variables $\{X_i\}_{i=1}^n$ and $\{Y_{i,j}\}_{j=1}^{K}$ for each $i=1,\ldots, n$ such that $X_i, Y_{i,1},\ldots, Y_{i,K}$ are exchangeable. 
Then for any $\alpha\in(0,\frac{2}{K+1}]$, the following inequality holds:
\begin{align*}
    \mathbb{E}\left[\frac{1}{n}\sum_{i=1}^n  \frac{X_i}{\alpha X_i + \frac{1-\alpha}{K}\sum_{i=1}^K Y_{i,j}}\right] \leq \frac{1}{\alpha(K+1)}
\end{align*}
Also, for any $\alpha\in\left[\frac{1}{K+1}, \frac{1}{2}\right] $, the following inequality holds:
\begin{align*}
    \mathbb{E}\left[\frac{1}{n}\sum_{i=1}^n  \frac{X_i}{\alpha X_i + \frac{1-\alpha}{K}\sum_{i=1}^K Y_{i,j}}\right] \leq 1
\end{align*}
\end{lemma}

\paragraph{\textit{Proof of Theorem~\ref{thm:4.2}.}}
First, remark that the DV bound of $\alpha$-skew KL divergence admits following:
\begin{align}
    D_{\tt{KL}}^{(\alpha)}(P\,\|\,Q) = \sup_{f\in\mathcal{F}}~\mathbb{E}_P[f] - \log\left(\alpha\mathbb{E}_P[e^f] + (1-\alpha)\mathbb{E}_Q[e^f]\right).
\end{align}
Then, note that $I_{\tt{CPC}}^{(\alpha)}(f)$ satisfies following:
\begin{align}\label{eqn:11}
\begin{split}
    I_{\tt{CPC}}^{(\alpha)}(f) &= \mathbb{E}_{(x,y)\sim P_{X,Y},\,y_i\sim P_Y, i=1,\ldots,K}\left[ \log\frac{e^{f(x,y)}}{\alpha e^{f(x,y)}+ 
    \frac{1-\alpha}{K}\sum_{i=1}^K e^{f(x,y_i)}} \right]\\
    &= \mathbb{E}_{(x,y)\sim P_{X,Y}}[f(x,y)] -\mathbb{E}_{(x,y)\sim P_{X,Y},\,y_i\sim P_Y, i=1,\ldots,K}\left[\log\left(\alpha e^{f(x,y)} + \frac{1-\alpha}{K} \sum_{i=1}^K e^{f(x,y_i)}\right)\right] \\
    &\geq \mathbb{E}_{P_{X,Y}}[f(x,y)] - \log \left(\alpha \mathbb{E}_{P_{X,Y}}[e^{f(x,y)}] + (1-\alpha) \mathbb{E}_{P_XP_Y}[e^{f(x,y)}] \right),
    \end{split}
\end{align}
where the last inequality comes from the Jensen's inequality that $-\mathbb{E}[\log X] \geq -\log\mathbb{E}[X]$. Then since \eqref{eqn:11} holds for all $f\in\mathcal{F}$, we have
\begin{align*}
    \sup_{f\in\mathcal{F}}I_{\tt{CPC}}^{(\alpha)}(f) \geq \sup_{f\in\mathcal{F}} \mathbb{E}_{P_{X,Y}}[f(x,y)] - \log \left(\alpha \mathbb{E}_{P_{X,Y}}[e^{f(x,y)}] + (1-\alpha) \mathbb{E}_{P_XP_Y}[e^{f(x,y)}] \right) = D_{\tt{KL}}^{(\alpha)}\left(P_{X,Y}\,\|\, P_XP_Y\right).
\end{align*}

Also, KL divergence also admits following variational form as known as NWJ objective~\citep{nwj}: 
\begin{align*}
    D_{\tt{KL}}(P\,\|\,Q) = \sup_{g\in\mathcal{F}}~ \mathbb{E}_P[g] +1 - \mathbb{E}_Q[e^g].
\end{align*}
Then by taking $g(x,y) = \log \frac{e^{f(x,y)}}{\alpha e^{f(x,y)} + \frac{1-\alpha}{K}\sum_{i=1}^K e^{f(x,y_i)}}$ for some sampled negatives $y_i\sim P_Y,i=1,\ldots, K$, in NWJ objective, we have following: 
\begin{align*}
    D_{\tt{KL}}^{(\alpha)}(P_{X,Y}\,\|\,P_XP_Y) &\geq \mathbb{E}_{y_i\sim P_Y, i=1,\ldots,K}\bigg[ \mathbb{E}_{P_{X,Y}}\bigg[\log \frac{e^{f(x,y)}}{\alpha e^{f(x,y)} + \frac{1-\alpha}{K}\sum_{i=1}^K e^{f(x,y_i)}}\bigg] +1 \\
    &-\alpha \mathbb{E}_{P_{X,Y}}\bigg[\frac{e^{f(x,y)}}{\alpha e^{f(x,y)} + \frac{1-\alpha}{K}\sum_{i=1}^K e^{f(x,y_i)}}\bigg]\\
    &-(1-\alpha) \mathbb{E}_{P_XP_Y}\bigg[\frac{e^{f(x,y)}}{\alpha e^{f(x,y)} + \frac{1-\alpha}{K}\sum_{i=1}^K e^{f(x,y_i)}}\bigg]\bigg]\\
    &\geq \mathbb{E}_{(x,y)\sim P_{X,Y}, y_i\sim P_Y, i=1,\ldots,K}\bigg[\log \frac{e^{f(x,y)}}{\alpha e^{f(x,y)} + \frac{1-\alpha}{K}\sum_{i=1}^K e^{f(x,y_i)}}\bigg] \\
    &= I_{\tt{CPC}}^{(\alpha)}(f),
\end{align*}
by Lemma \ref{lemma:1}, we have
\begin{align*}
    1-\alpha \mathbb{E}_{P_{X,Y}}\bigg[\frac{e^{f(x,y)}}{\alpha e^{f(x,y)} + \frac{1-\alpha}{K}\sum_{i=1}^K e^{f(x,y_i)}}\bigg]-(1-\alpha) \mathbb{E}_{P_XP_Y}\bigg[\frac{e^{f(x,y)}}{\alpha e^{f(x,y)} + \frac{1-\alpha}{K}\sum_{i=1}^K e^{f(x,y_i)}}\bigg]\bigg] \geq 1- \frac{1}{\alpha(K+1)}\geq 0,
\end{align*}
for $\alpha\in\left(0,\frac{2}{K+1}\right)$, and
\begin{align*}
    1-\alpha \mathbb{E}_{P_{X,Y}}\bigg[\frac{e^{f(x,y)}}{\alpha e^{f(x,y)} + \frac{1-\alpha}{K}\sum_{i=1}^K e^{f(x,y_i)}}\bigg]-(1-\alpha) \mathbb{E}_{P_XP_Y}\bigg[\frac{e^{f(x,y)}}{\alpha e^{f(x,y)} + \frac{1-\alpha}{K}\sum_{i=1}^K e^{f(x,y_i)}}\bigg]\bigg] \geq 1- 1\geq 0,
\end{align*}
for $\alpha\in\left[\frac{1}{K+1}, \frac{1}{2}\right]$.
Therefore, since we have
\begin{align*}
    D_{\tt{KL}}^{(\alpha)}(P\,\|\,Q) \leq \sup_{f\in\mathcal{F}} I_{\tt{CPC}}^{(\alpha)}(f) \leq D_{\tt{KL}}^{(\alpha)}(P\,\|\,Q),
\end{align*}
we have our results.
\subsection{Proof of Theorem~\ref{thm}}\label{appdx:pfthm32}
In this section, we state and define the regularity assumptions to derive the asymptotic upper bounds for the variance of $\alpha$-skew KL divergence in Theorem 3.2. 
We first review the definition of $f$-divergence. Let $f:(0,\infty)\rightarrow\mathbb{R}$ be a convex function with $f(1)=0$. Let $P$ and $Q$ be distributions with respect to a base measure $dx$ on domain $\mathcal{X}$, and assume $P\ll Q$. Then the $f$-divergence generated by $f$ is defined by
\begin{align*}
    D_f(P\|Q) \coloneqq \mathbb{E}_Q\left[f\left(\frac{dP}{dQ}\right)\right].
\end{align*}
Note that for any $c\in\mathbb{R}$, $D_{f_c}(P\|Q) = D_f(P\|Q)$, where $f_c(t) = f(t) +c(t-1)$. Hence, w.l.o.g, we assume $f(t)\geq 0$ for all $t\in(0,\infty)$. The conjugate of $f$ is a function $f^*:(0,\infty)\rightarrow[0,\infty)$ defined by $f^*(t) = t f(1/t)$, where $f^*(0)= \lim_{t\rightarrow0^+} f^*(t)$ for convenience. The conjugate $f^*$ of convex function $f$ is also convex. Also, one can see that $f^*(1)=0$ and $f^*(t)\geq 0$ for all $t\in(0,\infty)$, thus it induces another divergence $D_{f^*}$. The conjugate divergence $D_{f^*}$ satisfies $D_{f^*}(P\|Q) = D_f(Q\|P)$.

The KL divergence is a $f$-divergence generated by $f(t) = t\log t - t + 1$, and the $\alpha$-skew KL divergence is a $f$-divergence generated by 
\begin{align*}
f^{(\alpha)}(t)= t\log\left(\frac{t}{\alpha t + 1- \alpha}\right) - (1-\alpha)(t-1).    
\end{align*}
From~\cite{divfront}, we state following regularity assumptions on the functions $f$ and $f^*$.
\begin{assumption}\label{ass}
The generator $f$ is twice continuously differentiable with $f'(1)=0$. Moreover
\begin{enumerate}[noitemsep,topsep=0pt,label={\textbf{(A\arabic*})}]
    \item \label{asmp:fdiv:bounded} 
        We have $C_0 := f(0) < \infty$
        and $C_0^* := f^*(0) < \infty$.
    \item \label{asmp:fdiv:1st-deriv} 
        There exist constants $C_1, C_1^* < \infty$  
        such that for any $t \in (0, 1)$, we have, 
        \begin{align*}
            |f'(t)| &\leq C_1 \max\{1, \log(1/t)\}, \quad \text{and} \quad
            |(f^*)'(t)| \leq C_1^*  \max\{1, \log(1/t)\} .
        \end{align*}
    \item \label{asmp:fdiv:2nd-deriv} 
        There exist constants $C_2, C_2^* < \infty$ such that 
        for every $t \in (0, \infty)$, we have, 
        \begin{align*}
            \frac{t}{2} f''(t) \leq C_2, \quad \text{and} \quad
            \frac{t}{2} (f^*)''(t) \leq C_2^* \,.
        \end{align*}
\end{enumerate}
\end{assumption}
We refer authors \cite{divfront} for the detailed discussion on Assumption~\ref{ass}.
Then one can observe that KL divergence does not satisfy Assumption~\ref{ass}, because KL divergence can be unbounded.
On the other hand, the $\alpha$-skew KL divergence satisfies Assumption~\ref{ass} from following proposition.
\begin{proposition}[\cite{divfront}]\label{prop:const_skew}
    The $\alpha$-skew KL divergence generated by 
    $f^{(\alpha)}$ satisfies Assumption~\ref{ass} with 
    \[
    C_0 = 1-\alpha, \quad 
    C_0^* = \log\frac{1}{\alpha} - 1+\alpha, \quad
    C_1 = 1, \quad 
    C_1^* = \frac{(1-\alpha)^2}{\alpha},\quad
    C_2 = \frac{1}{2}, \quad
    C_2^* = \frac{1-\alpha}{8 \alpha} \,.
    \]
\end{proposition}
For the general $f$-divergences which satisfy Assumption~\ref{ass}, the following concentration bound holds.
\begin{proposition}[\cite{divfront}]\label{prop:concenbound}
Assume $f$ satisfies Assumption~\ref{ass}, and let $P$ and $Q$ be two distributions with $P \ll Q$.
Let $P_m$ be $m$ i.i.d samples from $P$ and $Q_n$ be $n$ i.i.d samples from $Q$.
Then the $f$-divergence $D_f$ satisfies following:
\begin{align*}
    \mathbb{P}[\,|D_f(P_m \| Q_n) - \mathbb{E}[D_f(P_m \| Q_n)] | > \varepsilon \,] \leq 2\exp\left(-\frac{\varepsilon^2}{\frac{2}{m}(C_1\log m + c_1)^2 + \frac{2}{n}(C_1^*\log n + c_2)^2}\right)
\end{align*}
where $c_1 =\max\{C_0^*,C_2\}$ and $c_2 = \max\{C_0,C_2^*\}$.
\end{proposition}
Thus, the following lemma derives a concentration bound for the $\alpha$-skew KL divergence by plugging the constants in Proposition~\ref{prop:const_skew} to Proposition~\ref{prop:concenbound}.
\begin{lemma}
For $\alpha<\frac{1}{8}$, the following holds:
\begin{align*}
    \mathbb{P}[\,|D_{\text{KL}}^{(\alpha)}(P_m \| Q_n) - \mathbb{E}[D_{\text{KL}}^{(\alpha)}(P_m \| Q_n)] | > \varepsilon\,] \leq 2\exp\left( - \frac{\varepsilon^2}{\frac{2}{m}\log^2(\alpha m) + \frac{2}{\alpha^2n}\log^2(e^{1/8}n)}\right)
\end{align*}
\end{lemma}
\begin{proof}
Note that $C_0^* = \log(1/\alpha) -1 + \alpha \geq C_2 = 1/2 $, and $C_0=1-\alpha \leq C_2^*=\frac{1-\alpha}{8\alpha}$ for $\alpha < \frac{1}{8}$. Then the concentration bound follows from Proposition~\ref{prop:concenbound}.
\end{proof}

Lastly, we present following upper bound on the bias of empirical estimator of KL divergence:
\begin{proposition}[\cite{rubenstein2019practical}]\label{prop:bias}
Suppose $P\ll Q$, and ${\tt{Var}}[dP/dQ]< \infty$. Then we have
\begin{align*}
    |\mathbb{E}[D_{\text{KL}}(P_m\|Q_n)] - D_{\text{KL}}(P\|Q) | \leq
    \frac{\chi^2(P\|Q)}{\min\{n,m\}}.
\end{align*}
\end{proposition}

From proposition~\ref{prop:bias}, we have 
\begin{align*}
    |\mathbb{E}[D_{\text{KL}}^{(\alpha)}(P_m \| Q_n)] - D_{\text{KL}}^{(\alpha)}(P \| Q)| \leq \frac{\chi^2(P\|\alpha P + (1-\alpha)Q)}{\min\{n,m\}},
\end{align*}
where $\chi^2(P \| \alpha P + (1-\alpha)Q) = \int \frac{d^2P}{\alpha dP + (1-\alpha)dQ}\leq \int \frac{1}{\alpha}dP = \frac{1}{\alpha}$, or 
$\int \frac{d^2P}{\alpha dP + (1-\alpha)dQ} \leq \frac{1}{1-\alpha}\int \frac{d^2P}{dQ} = \frac{\chi^2(P\|Q)}{1-\alpha}$. Therefore, we have
\begin{align*}
    |\mathbb{E}[D_{\text{KL}}^{(\alpha)}(P_m \| Q_n)] - D_{\text{KL}}^{(\alpha)}(P \| Q)| \leq \frac{c(\alpha)}{\min\{n,m\}},\quad \text{for}\quad c(\alpha)\coloneqq \min\left\{\frac{1}{\alpha}, \frac{\chi^2(P\|Q)}{1-\alpha}\right\}.
\end{align*}

Now we present the proof of Theorem~\ref{thm} in the main paper.

\vspace{0.05in}
\textit{Proof of Theorem~\ref{thm}.}
Define 
\begin{align*}
    B_1 &\coloneqq |D_{\text{KL}}^{(\alpha)}(P_m \| Q_n) - \mathbb{E}[D_{\text{KL}}^{(\alpha)}(P_m \| Q_n)]| \\
    B_2 &\coloneqq |\mathbb{E}[D_{\text{KL}}^{(\alpha)}(P_m \| Q_n)] - D_{\text{KL}}^{(\alpha)}(P \| Q)| \\
    B_3 &\coloneqq |\mathcal{I}_{\text{KL}}^{(\alpha)}(\widehat{f}) - D_{\text{KL}}^{(\alpha)}(P_m \| Q_n) |.
\end{align*}
From Proposition~\ref{prop:bias}, we have $B_2 \leq \frac{c_2}{\min\{n,m\}}$ for some constant $c(\alpha)>0$, and from the assumption, we have $B_3 \leq \varepsilon_f$.
By using triangle inequality twice, we have
\begin{align*}
    B_1 &\geq |D_{\text{KL}}^{(\alpha)}(P_m \| Q_n) - D_{\text{KL}}^{(\alpha)}(P \| Q)| - B_2 \\
    &\geq |D_{\text{KL}}^{(\alpha)}(P \| Q) -\mathcal{I}_{\text{KL}}^{(\alpha)}(\widehat{f})| - B_3 - B_2.
\end{align*}
Therefore, it follows that
\begin{align*}
    \mathbb{P}[ |D_{\text{KL}}^{(\alpha)}(P \| Q) -\mathcal{I}_{\text{KL}}^{(\alpha)}(\widehat{f})| > \varepsilon ] &\leq 
    \mathbb{P}[ B_1 + B_2 + B_3 > \varepsilon ] \\
    &\leq \mathbb{P}\left[ B_1 > \varepsilon - \varepsilon_f - \frac{c(\alpha)}{\min\{n,m\}}\right]\\
    &\leq 2\exp\left( -\frac{(\varepsilon - \varepsilon_f - c(\alpha) / \min\{n,m\})^2}{\frac{2}{m}\log^2(\alpha m) + \frac{2}{\alpha^2n}\log^2(e^{1/8}n)}\right),
\end{align*}
where the last inequality comes from Proposition~\ref{prop:concenbound}.
Since the following holds for any random variable $X$,
$${\tt{Var}}(X) = \mathbb{E}[(X-\mathbb{E}X)^2] = \int_0^\infty \mathbb{P}[|X-\mathbb{E}X|^2 >t]dt =\int_0^\infty \mathbb{P}[|X-\mathbb{E}X| >\sqrt{t}]dt,$$
we have
\begin{align*}
    {\tt{Var}}_{P,Q}\big[\mathcal{I}_{\text{KL}}^{(\alpha)}(\widehat{f})\big] \leq \int_0^\infty 2\exp\left( -\frac{(\sqrt{t} - \varepsilon_f - c(\alpha) / \min\{n,m\})^2}{\frac{2}{m}\log^2(\alpha m) + \frac{2}{\alpha^2n}\log^2(e^{1/8}n)}\right)dt,
\end{align*}
which proves our result.

\subsection{Proof sketch for R\'enyi divergence}\label{thmrenyi}
In this section, we provide a proof sketch for the upper bound for the variance of RMLCPC objective. We use similar proof technique that we used in previous section.
Recall that the R\'enyi Divergence is related to $\gamma$-divergence\footnote{we use $\gamma$ for consistency, in general, it is called $\alpha$-divergence.} $D_\gamma$, a $f$-divergence generated by $f^{(\gamma)}(t) = \frac{t^\gamma -1}{\gamma(\gamma-1)} $. 
The KL divergence is recovered from $\gamma$-divergence if we let $\gamma\rightarrow1$. 
Then we have following equation for R\'enyi divergence and $\gamma$-divergence:
\begin{align*}
    R_\gamma(P\|Q) = \frac{1}{\gamma(\gamma-1)}\log (\gamma(\gamma-1)D_\gamma(P\|Q)+1).
\end{align*}
One can observe that $f(x) = \frac{1}{\gamma(\gamma-1)}\log(\gamma(\gamma-1)x + 1)$ is of 1-Lipschitz function for any $\gamma \neq 0,1$. 
Thus, if we can derive concentration bound for $\alpha$-skew $\gamma$-divergence, we can derive concentration bound for R\'enyi divergence.

Here, we provide an example when $\gamma=2$. Remark that when $\gamma=2$, the $\gamma$-divergence is equivalent to $\chi^2$-divergence, which is generated by $f(t) = (t-1)^2$. 
Then the $\alpha$-skew $\chi^2$-divergence is generated by following generators:
\begin{align*}
    f^{(\alpha)}(t) = \frac{(t-1)^2}{\alpha t + 1-\alpha}= (f^{(1-\alpha)})^*(t).
\end{align*}
Then from~\cite{divfront}, the $\alpha$-skew $\chi^2$-divergence satisfies Assumption~\ref{ass} with
\[
C_0 = \frac{1}{1-\alpha}, \quad 
C_0^* = \frac{1}{\alpha}, \quad
C_1 = \frac{2}{(1-\alpha^2}, \quad 
C_1^* = \frac{2}{\alpha^2},\quad
C_2 = \frac{4}{27\alpha(1-\alpha)^2}, \quad
C_2^* = \frac{4}{27\alpha^2(1-\alpha)} \,.
\]
Then by using Proposition~\ref{prop:concenbound}, for $\alpha < \min\{ \frac{2}{3\sqrt{3}}, 1-\frac{2}{3\sqrt{3}} \}$, we have
\begin{align*}
    \mathbb{P}[ |D_{\chi^2}^{(\alpha)}(P_m \| Q_n) - \mathbb{E}[D_{\chi^2}^{(\alpha)}(P_m \| Q_n)] | > \varepsilon ] \leq 
    2\exp\left( - \frac{\varepsilon^2}{\frac{c_1}{n}\log^2(c_2 n) + \frac{c_3}{\alpha^2m}\log^2(c_4m)}\right),
\end{align*}
for some constant $c_1,c_2,c_3,c_4 >0$, then from the Lipschitz continuity, we have
\begin{align*}
    \mathbb{P}[ |R_{2}^{(\alpha)}(P_m \| Q_n) - \mathbb{E}[R_{2}^{(\alpha)}(P_m \| Q_n)] | > \varepsilon ] \leq 
    2\exp\left( - \frac{\varepsilon^2}{\frac{c_1}{n}\log^2(c_2 n) + \frac{c_3}{\alpha^2m}\log^2(c_4m)}\right),
\end{align*}
Thus, by deriving the asymptotic bound for the bias $|\mathbb{E}[R_{2}^{(\alpha)}(P_m\|Q_n)] - R_{2}^{(\alpha)}(P\|Q)| $, and if the R\'enyi variational objective has sufficiently small error, we can derive upper bound for the variance of RMLCPC objective.

\section{Details for Experiments}
\subsection{Implementation}\label{appendix:generalderivation}
As explained in Section 3.3, we derive the full equivalent form of MLCPC and RMLCPC with non-zero $\alpha$.
Recall that the $\alpha$-MLCPC objective with neural network $f_\theta$ is given by
\begin{align}\label{axeq:mlcpc}
    \mathcal{I}_{\tt{MLCPC}}^{(\alpha)}(f_\theta) = \mathbb{E}_{v,v^+}[f_\theta(v,v^+)] - \log\big(\alpha\mathbb{E}_{v,v^+}[e^{f_\theta(v,v^+)}] + (1-\alpha)\mathbb{E}_{v,v^-}[e^{f_\theta(v,v^-)}]\big),
\end{align}
then the gradient of \eqref{axeq:mlcpc} with respect to parameter $\theta$ is given by
\begin{align*}
    &\nabla_\theta\mathcal{I}_{\tt{MLCPC}}^{(\alpha)}(f_\theta) \\
    &= \mathbb{E}_{v,v^+}[\nabla_\theta f_\theta(v,v^+)] - \frac{\alpha \mathbb{E}_{v,v^+}[e^{f_\theta(v,v^+)}\nabla_\theta f_\theta(v,v^+)] + (1-\alpha) \mathbb{E}_{v,v^-}[e^{f_\theta(v,v^-)}\nabla_\theta f_\theta(v,v^-)]}{\alpha \mathbb{E}_{v,v^+}[e^{f_\theta(v,v^+)}] + (1-\alpha) \mathbb{E}_{v,v^-}[e^{f_\theta(v,v^-)}]} \\
    &= \mathbb{E}_{v,v^+}[\nabla_\theta f_\theta(v,v^+)] - \big(\mathbb{E}_{{\tt{sg}}(q_\theta(v,v^+))}[\nabla_\theta f_\theta(v,v^+)] +\mathbb{E}_{{\tt{sg}}(q_\theta(v,v^-))}[\nabla_\theta f_\theta(v,v^-)]\big),
\end{align*}
where
\begin{align*}
    q_\theta(v,v^+) &\propto \frac{\alpha p(v,v^+)e^{f_\theta(v,v^+)}}{\alpha \mathbb{E}_{v,v^+}[e^{f_\theta(v,v^+)}] + (1-\alpha) \mathbb{E}_{v,v^-}[e^{f_\theta(v,v^-)}]} \\
    q_\theta(v,v^-) &\propto \frac{(1-\alpha)p(v)p(v^-) e^{f_\theta(v,v^-)}}{\alpha \mathbb{E}_{v,v^+}[e^{f_\theta(v,v^+)}] + (1-\alpha) \mathbb{E}_{v,v^-}[e^{f_\theta(v,v^-)}]},
\end{align*}
are importance weights for positive and negative pairs in the second term of \eqref{axeq:mlcpc}. For the generalized $(\alpha,\gamma)$-RMLCPC, recall that the RMLCPC with neural network $f_\theta$ is given by
\begin{align}\label{axeq:rmlcpc}
    \mathcal{I}_{\tt{RMLCPC}}^{(\alpha,\gamma)}(f_\theta) = \frac{1}{\gamma-1}\log\mathbb{E}_{v,v^+}[e^{(\gamma-1)f_\theta(v,v^+)}] 
    - \frac{1}{\gamma}\log\big(\alpha\mathbb{E}_{v,v^+}[e^{\gamma f_\theta(v,v^+)}] + (1-\alpha)\mathbb{E}_{v,v^-}[e^{\gamma f_\theta(v,v^-)}]\big).
\end{align}
Then the gradient of \eqref{axeq:rmlcpc} with respect to $\theta$ is given by
\begin{align*}
    &\nabla_\theta\mathcal{I}_{\tt{RMLCPC}}^{(\alpha,\gamma)}(f_\theta) \\
    &= \mathbb{E}_{q_\theta^{(1)}(v,v^+)}[\nabla_\theta f_\theta(v,v^+)] 
    - \frac{\alpha \mathbb{E}_{v,v^+}[e^{\gamma f_\theta(v,v^+)}\nabla_\theta f_\theta(v,v^+)] + (1-\alpha) \mathbb{E}_{v,v^-}[e^{\gamma f_\theta(v,v^-)}\nabla_\theta f_\theta(v,v^-)]}{\alpha \mathbb{E}_{v,v^+}[e^{\gamma f_\theta(v,v^+)}] + (1-\alpha) \mathbb{E}_{v,v^-}[e^{\gamma f_\theta(v,v^-)}]} \\
    &= \mathbb{E}_{{\tt{sg}}(q_\theta^{(1)}(v,v^+))}[\nabla_\theta f_\theta(v,v^+)]  - \big(\mathbb{E}_{{\tt{sg}}(q_\theta^{(2)}(v,v^+))}[\nabla_\theta f_\theta(v,v^+)] +\mathbb{E}_{{\tt{sg}}(q_\theta^{(2)}(v,v^-))}[\nabla_\theta f_\theta(v,v^-)]\big),
\end{align*}
where
\begin{align*}
    q_\theta^{(1)}(v,v^+) &\propto p(v,v^+)e^{(\gamma-1)f_\theta(v,v^+)}\\
    q_\theta^{(2)}(v,v^+) &\propto \frac{\alpha p(v,v^+)e^{\gamma f_\theta(v,v^+)}}{\alpha \mathbb{E}_{v,v^+}[e^{\gamma f_\theta(v,v^+)}] + (1-\alpha) \mathbb{E}_{v,v^-}[e^{\gamma f_\theta(v,v^-)}]} \\
    q_\theta^{(2)}(v,v^-) &\propto \frac{(1-\alpha)p(v)p(v^-) e^{\gamma f_\theta(v,v^-)}}{\alpha \mathbb{E}_{v,v^+}[e^{\gamma f_\theta(v,v^+)}] + (1-\alpha) \mathbb{E}_{v,v^-}[e^{\gamma f_\theta(v,v^-)}]},
\end{align*}
are importance weights for positive and negative pairs for the first and second term in \eqref{axeq:rmlcpc}. 
As discussed in Section 3.3, regardless of the insertion of parameter $\alpha$, MLCPC and RMLCPC explicitly conduct hard-negative sampling, and especially RMLCPC, conducts easy-positive sampling. 

The PyTorch style pseudo-code for our implementation on R\'enyiCL is demonstrated in Algorithm~\ref{alg:code}.
In our default implementation, we use multi-crop data augmentation~\cite{swav} with two global views and multiple local views. 
In Algorithm~\ref{alg:code}, we implement the negative of $(\alpha,\gamma)$-RMLCPC, and $\alpha$-MLCPC which is equivalent to the case when $\gamma\rightarrow1$ for RMLCPC, for the contrastive losses.
Note that our pseudo-code for contrastive objectives are based on the approach that uses importance weights.

\begin{algorithm}[ht]
\caption{R\'enyiCL: PyTorch-like Pseudocode}
\label{alg:code}
\algcomment{
}
\definecolor{codeblue}{rgb}{0.25,0.5,0.5}
\definecolor{codekw}{rgb}{0.85, 0.18, 0.50}
\begin{lstlisting}[language=python]
# f_q: base encoder: backbone + proj mlp + pred mlp
# f_k: momentum encoder: backbone + proj mlp
# m: momentum coefficient
# tau: temperature
# aug_g : global view data augmentation 
# aug_l : local view data augmentation

for x in loader:  # load a minibatch x with N samples
    x1, x2 = aug_g(x), aug_g(x)  # two global views
    k1, k2 = f_k(x1), f_k(x2) # keys: [N, D] each
    q1, q2 = [], [] # list of queries
    q1.append(f_q(x1)), q2.append(f_q(x2)) # queries: [N, D] each
    for i in range(n_crops):
        x_l = aug_l(x)
        q1.append(f_q(x_l)), q2.append(f_q(x_l))  # pass local views only to base encoder
    
    pos1, pos2, neg1, neg2 = [], [], [], []
    for j in range(n_crops + 1):
        pos, neg = extract_pos_neg(q1[j], k2)  # extract pos, neg from q1 and k2
        pos1.append(pos)
        neg1.append(neg)
        
        pos, neg = extract_pos_neg(q2[j], k1)  # extract pos, neg from q2 and k1
        pos2.append(pos)
        neg2.append(neg)
    pos1, pos2, neg1, neg2 = pos1.cat(), pos2.cat(), neg1.cat(), neg2.cat()
    loss = ctr_loss(pos1, neg1) + ctr_loss(pos2, neg2) # symmetrized

    update(f_q)  # optimizer update: f_q
    f_k = m*f_k + (1-m)*f_q  # momentum update: f_k

# extract positive pairs and negative pairs
def extract_pos_and_neg(q, k):
    # N : number of samples in minibatch
    logits = mm(q, k.t()) / tau  # matrix multiplication
    pos = logits.diag()  # size N
    neg = logits.flatten()[1:].view(N-1, N+1)[:,:-1].reshape(N,N-1)  # size N x (N-1)
    return pos, neg

# contrastive losses using importance weights
def ctr_loss(pos, neg, alpha, gamma):
    pos_d = pos.detach(), neg_d = neg.detach() # no gradient for importance weights
    if gamma == 1:  # MLCPC
        loss_1 = -1 * pos.mean()
        iw2_p, iw2_n = pos_d.exp(), neg_d.exp()  # importance weights
        loss_2 = alpha*(pos*iw2_p.exp()).mean()+(1-alpha)*(neg*iw2_n.exp()).mean()
        loss_2 /= alpha*iw2_p.mean() + (1-alpha)*iw2_n.mean()
        loss = loss_1 + loss_2
    elif gamma > 1:  # RMLCPC
        iw1 = ((gamma - 1)*pos_d).exp()  # importance weight for first term
        loss_1 = (pos*iw1).mean() / iw1.mean()
        iw2_p, iw2_n = (gamma*pos_d).exp(), (gamma*neg_d).exp() # importance weights
        loss_2 = alpha*(pos*iw2_p).mean() + (1-alpha)*(neg*iw2_n).mean()
        loss_2 /= alpha*iw2_p.mean() + (1-alpha)*iw2_n.mean()
        loss = loss_1 + loss_2
    return loss
\end{lstlisting}
\end{algorithm}

\paragraph{Effect of $\alpha$ in R\'enyiCL.}\label{appendix:alphaexp}
One can observe that the gradient of positive pairs in the second term is multiplied by $\alpha$, thus the gradient of positive pairs is affected by $\alpha$ and negative pairs. 
However, since we use a small value of $\alpha$, the gradient of the second term can be ignored in practice. To verify our claim, we experimented with different values of $\alpha$ on both basic and hard data augmentations: 
\begin{table}[t]
\centering
\small
\caption{Ablation on the effect of $\alpha$ when training with harder data augmentation.}\label{tab:alphaexp}
\vspace{0.02in}
\begin{tabular}{l llll}
\toprule
$\alpha^{-1}$&  1024& 4096& 16384& 65536 \\
\midrule
Base Aug. & 79.0 & 79.3 & 78.6 & 78.4 \\
Hard Aug. & 81.1 & 81.3 & {\bf81.6} & 81.1 \\
\midrule
Gap       & +2.1 & +2.0 & {\bf+3.0} & +2.7 \\
\bottomrule
\end{tabular}
\end{table}

In Table~\ref{tab:alphaexp}, one can observe that as $\alpha$ becomes smaller, the gap between using hard augmentation and the base augmentation becomes larger. Thus, there is a tradeoff in the choice of $\alpha$: the $\alpha$ should be large enough to evade large variance (Thm 3.2), but small $\alpha$ is preferred to have the effect of easy positive sampling.

\subsection{Further ablation studies}
In this section, we provide further ablation studies to validate the effectiveness of R\'enyiCL. 
\paragraph{Fair comparison with other self-supervised methods.}
To see the effectiveness of R\'enyiCL, we conduct additional experiments for fair comparison with other self-supervised methods.
MoCo v3~\cite{mocov3} is a state-of-the-art method in contrastive self-supervised learning. For fair comparison of R\'enyiCL with MoCo v3, we apply harder data augmentations such as RandAugment~\cite{ra} and multi-crops~\cite{swav} on MoCo v3. 
We first reimplemented MoCo v3 with smaller batch size (original 4096 to 1024) and in our setting, we achieved 69.6\%, which is better than their original reports (68.9\%). 
Then we use harder data augmentations RandAugment~(RA), RandomErasing~(RE), and multi-crops~(MC) when training MoCo v3. 
Second, we compare R\'enyiCL with DINO~\cite{dino}. Note that DINO uses default multi-crops, therefore we further applied RandAugment (RA) and RandomErasing (RE) for a fair comparison. 
For each model, we train for 100 epochs.
\begin{table}[t]
\centering
\small
\caption{Fair comparison with other self-supervised learning methods by using same hard data augmentation. All models are run by us with 100 training epochs. }\label{tab:mocodino}
\vspace{0.02in}
\begin{tabular}{l ccc}
\toprule
Method & DINO~\cite{dino} & MoCo v3~\cite{mocov3} & R\'enyiCL \\
\midrule
Base & 70.9 & 69.6 & 69.4 \\
Hard & 72.5 & 73.5 & {\bf74.3} \\
\midrule
Gain & +1.6 & +3.9 & {\bf+4.9} \\
\bottomrule
\end{tabular}
\end{table}
The results are shown in Table~\ref{tab:mocodino}. One can observe that while MoCo v3 and DINO benefit from using harder data augmentation, RenyiCL shows the best performance when using harder data augmentation. Also, compared to the Base data augmentation, RényiCL attains the most gain by using harder data augmentation. 
\begin{table}[t]
\centering
\begin{minipage}[t]{.49\linewidth}
\centering
\small
\caption{Ablation on the robustness of R\'enyiCL to Noisy Crop.}
\begin{tabular}[b]{@{}l ccc@{}}
    \toprule
    Method & CIFAR-10   & CIFAR-100  & ImageNet-100 \\
    \midrule
    CPC    & 91.7       & 65.4       & 79.7  \\
    +NC    & 90.6(-1.1) & 64.9(-0.5) & 79.2(-0.5) \\
    \midrule
    MLCPC  & 91.9       & 65.6       & 79.5  \\
    +NC    & 90.7(-1.2) & 65.0(-0.6) & 79.4(-0.1) \\
    \midrule
    RMLCPC & 90.7       & 64.5       & 78.9  \\
    +NC    & {\bf 91.6(+0.9)} & {\bf 67.6(+3.1)} & {\bf80.5(+1.6)} \\
\bottomrule
\label{tab:noisycrop}
\end{tabular}
\end{minipage}
\hspace{\fill}%
\begin{minipage}[t]{0.43\linewidth}
\centering
\small
\caption{Ablation on the robustness of R\'enyiCL when learning with limited data augmentation.}
\begin{tabular}[b]{@{}l ccc @{}}
\toprule
Method & CIFAR-10 & CIFAR-100 & STL-10 \\
\midrule
CPC    & 63.2     & 31.8      & 52.8   \\
MLCPC  & 63.3     & 32.9      & 53.1   \\
RMLCPC & {\bf67.1}     & {\bf41.2}      & {\bf56.4} \\
\bottomrule
\label{tab:noaug}
\end{tabular}
\end{minipage}
\end{table}

\paragraph{Robustness to data augmentations.}
To see the robustness of R\'enyiCL on the choice of data augmentation, we conduct additional ablation studies on the robustness of R\'enyiCL on various data augmentation schemes. 
We first show that R\'enyiCL is more robust to the noisy data augmentation, as similar to that in~\cite{rince}, we apply Noisy Crop, which additionally perform RandomResizedCrop of size 0.2.
Then the generated views might contain only nuisance information (e.g. background), thus it can hurt the generalization of contrastive learning. 
We train each model with CPC, MLCPC, and RMLCPC objectives and data augmentations with and without Noisy Crop. We experimented on CIFAR-10, CIFAR-100, and ImageNet-100 datasets and the results are shown in Table~\ref{tab:noisycrop}. One can observe that while the performance of using CPC and MLCPC degrades as we apply the noisy crop, RMLCPC conversely improves the performance. Hence, RMLCPC is robust in the choice of data augmentation.

Second, we experiment when there is limited data augmentation. For each CIFAR-10, CIFAR-100, and STL-10 dataset, we remove data augmentations such as color jittering and grayscale and only apply RandomResizedCrop and HorizontalFlip. The results are in Table~\ref{tab:noaug}.
Note that RMLCPC achieves the best performance among CPC and MLCPC, when we do not apply data augmentation. Therefore we show that the RMLCPC objective is not only robust to the hard data augmentation but is also robust when there is limited data augmentation.

\subsection{ImageNet experiments}\label{appdx:dataset}
\paragraph{Data augmentation.}
For the baseline, we follow the good practice in existing works~\cite{instdisc,simclr, byol, mocov3}, which includes random resized cropping, horizontal flipping, color jittering~\cite{instdisc}, grayscale conversion~\cite{instdisc}, Gaussian blurring~\cite{simclr}, and solarization~\cite{byol}. 
Given the baseline data augmentation, we add the following data augmentation methods:
\begin{itemize}
    \item RandAugment~\cite{ra}: a strong data augmentation strategy that includes 14 image-specific transformations such as AutoContrast, Posterize, etc. We use default settings that were used in~\cite{ra}~(2 operations at each iteration and a magnitude scale of 10).
    \item RandomErasing~\cite{re}: a simple data augmentation strategy that randomly masks a box in an image with a random value. The RandomErasing was applied with probability $0.2$.
    \item Multi-crop~\cite{swav}: a multi-view data augmentation strategy for self-supervised learning, we use two global crops of size 224 and 2 or 6 local crops of size 96. For global crops, we scale the image with $(0.2,1.0)$ and for the local crops, we scale the image with $(0.05,0.2)$, which is slightly different from that in~\cite{swav}.
\end{itemize}
For the implementation of RandAugment and RandomErasing, we use the implementation in~\cite{rw2019timm}.
Note that when using Multi-crop, we only use RandAugment on the global views, and we do not use RandomErasing, which slightly degrades performance when used for Multi-crops.
Otherwise, if we do not use Multi-crop, we use both RandAugment and RandomErasing. 

\paragraph{Model.}
We use ResNet-50~\cite{resnet} for backbone. 
For our base encoder, we further use a projection head~\cite{simclr}, and a prediction head~\cite{byol}.
The momentum encoder is updated by the moving average of the base encoder, except that we do not use the prediction head for momentum encoder~\cite{mocov3}.
The momentum coefficient is initialized by $0.99$ and gradually increased by the half-cycle cosine schedule to 1.
We use 2-layer MLP with dimensions 2048-4096-256 for the projection head. We attach the batch-normalization~(BN) layer and ReLU activation layer at the end of each fully-connected layer, except for the last one, to which we only attach the BN layer.
For the prediction head, we use 2-layer MLP with dimensions of 256-4096-256, and we attach BN and ReLU at the end of each fully connected layer except the last one.
Then the critic is implemented by the temperature-scaled cosine similarity: $f(x,y;\tau) = x^\top y/ (\tau \|x\|_2\|y\|_2)$, with $\tau=0.5$. 
For RMLCPC objective, we use $\gamma=2.0$, and $\alpha=1/65536$.
\paragraph{Optimization.}
We use LARS~\cite{lars} optimizer with a batch size of 512, weight decay of $1.5\times10^{-6}$, and momentum of $0.9$ for all experiments.
We use root scaling for learning rate~\cite{simclr}, which was effective in our framework. 
When training for 100 epochs, we use the base learning rate of 0.15, and for training for 200 epochs, we use the base learning rate of 0.075.
The learning rate is gradually annealed by a half-cycle cosine schedule.
We also used automatic mixed precision for scalability.
\paragraph{Linear evaluation.}
For ImageNet linear evaluation accuracy, we report the validation accuracy of the linear classifier trained on the top of frozen features. 
We use random resized cropping of size 224 and random horizontal flipping for training data augmentation, and resizing to the size of 256 and center-crop of 224 for validation data augmentation.
We use cross-entropy loss for training, where we train for 90 epochs with SGD optimizer, batch size of 1024, and base learning rate of 0.1, where we use linear learning rate scaling and learning rate follows cosine annealing schedule. We do not use weight decay.
\paragraph{Semi-supervised learning.}
For semi-supervised learning on ImageNet, we follow the protocol as in~\cite{simclr, byol, swav}. 
We add a linear classifier on the top of frozen representation, then train both backbone and classifier using either 1\% or 10\% of the ImageNet training data. 
For a fair comparison, we use the splits proposed in~\cite{simclr}.
We use the same data augmentation and loss function as in linear evaluation.
As same as previous works, we use different learning rates for each backbone and classification head.
For the 1\% setting, we use a base learning rate of 0.01 for backbone and 0.2 for classification head.
For the 10\% setting, we use a base learning rate of 0.08 for the backbone and 0.03 for the classification head.
The learning rate is decayed with a factor of 0.2 at 16 and 20 epochs.
We do not use weight decay or other regularization technique.
We train for 20 epochs with SGD optimizer, batch size of 512 for both 1\% and 10\% settings, and do not use weight decay.

\begin{table}[t]
\centering
\small
\caption{Dataset information for the transfer learning tasks. For FC100, CUB200, Plant Disease, we perform few-shot learning, otherwise, we perform linear evaluation.}\label{tab:dataset_info}
\vspace{0.02in}
\resizebox{\textwidth}{!}{
\begin{tabular}{llllll}
\toprule
Dataset & \# of classes & Training & Validation & Test & Metric \\ \midrule
CIFAR10 \citep{cifar}         &  10 & 45000 & 5000 & 10000 & Top-1 accuracy \\
CIFAR100 \citep{cifar}        & 100 & 45000 & 5000 & 10000 & Top-1 accuracy \\
Food    \citep{food}       & 101 & 68175 & 7575 & 25250 & Top-1 accuracy \\
MIT67 \citep{mit67}             &  67 &  4690 &  670 &  1340 & Top-1 accuracy \\
Pets \citep{pets}                 &  37 &  2940 &  740 &  3669 & Mean per-class accuracy \\
Flowers \citep{flowers} & 102 &  1020 & 1020 &  6149 & Mean per-class accuracy \\
Caltech101 \citep{caltech}        & 101 &  2525 &  505 &  5647 & Mean Per-class accuracy \\
Cars \citep{cars}            & 196 &  6494 & 1650 &  8041 & Top-1 accuracy \\
Aircraft \citep{aircraft}      & 100 &  3334 & 3333 &  3333 & Mean Per-class accuracy \\
DTD (split 1) \citep{dtd}         & 47  &  1880 & 1880 &  1880 & Top-1 accuracy \\
SUN397 (split 1) \citep{sun397}        & 397 & 15880 & 3970 & 19850 & Top-1 accuracy \\
VOC2007~\cite{voc2007}         & 20 & 4952  & 2501 & 2510 & Mean average precision \\ \midrule
FC100 \citep{fc100}             &  20 & - & - & 12000 & Average accuracy \\
CUB200 \citep{ra}              & 200 & - & - & 11780 & Average accuracy \\
Plant Disease \citep{plant}      &  38 & - & - & 54305 & Average accuracy \\
\bottomrule
\end{tabular}}
\end{table}

\paragraph{Transfer learning.}
For transfer learning, we train a classifier on the top of frozen representations as done in many previous works~\cite{simclr,byol}. 
In Table~\ref{tab:dataset_info}, we list the information of the datasets we used; the number of classes, number of samples for each train/val/test split, and evaluation metric. 
We use the same data augmentation that we used for linear evaluation on the ImageNet dataset.
For optimization, we $\ell_2$-regularized L-BFGS, where the regularization parameter is selected from a range of 45 logarithmically spaced values from $10^{-6}$ to $10^5$ using the validation split.
After the best hyperparameter is selected, we train the linear classifier using both training and validation splits and report the test accuracy using the metric instructed in Table~\ref{tab:transfer}.
The maximum number of iterations in L-BFGS is 5000 and we use the previous solution as an initial point, i.e., a warm start, for the next step.
For few-shot learning experiments, we perform logistic regression on the top of frozen representations and use $N\times K$ support samples without fine-tuning and data augmentation in a $N$-way $K$-shot episode. 
\paragraph{Extended version of Table~4.}
We provide additional experimental results in the comparison with other contrastive representation learning methods that use stronger augmentations.
The InfoMin~\cite{infomin} used RandAugment~\cite{ra} and Jigsaw cropping for data augmentation and trained for 800 epochs.
The CLSA~\cite{clsa} used both RandAugment~\cite{ra} and pyramidal multi-crop data augmentation, which generates additional crops of size $192\times 192, 160\times 160, 128\times 128 $, and $96 \times 96$. 
For comparison, we use the checkpoints from their official github repositories\footnote{InfoMin: \url{https://github.com/HobbitLong/PyContrast}}
\footnote{CLSA: \url{https://github.com/maple-research-lab/CLSA}}.
In Table~\ref{tab:fewshotext}, we report the few-shot classification accuracies~(\%) on FC100~\cite{fc100}, CUB200~\cite{ra}, and Plant Disease~\cite{plant} datasets, which is the extended version of Table~3 in the main paper.
In Table~\ref{tab:linevalext}, we report the linear evaluation accuracies~(\%) on various object classification datasets extending the results in Table~3 in the main paper.
\begin{table}[t]
\small
\caption{Comparison of contrastive self-supervised methods with harder augmentations by downstream few-shot classification accuracy~(\%) over 2000 episodes on FC100, CUB200, and Plant disease datasets. $(N,K)$ denotes $N$-way $K$-shot classification. $\dag$ denotes the usage of multi-crop data augmentation. {\bf Bold} entries denote the best performance.}
\centering
\begin{tabular}[t]{l cc c cc c cc}
    \toprule
          & \multicolumn{2}{c}{FC100} && \multicolumn{2}{c}{CUB200} && \multicolumn{2}{c}{Plant Disease} \\
          \cmidrule{2-3}\cmidrule{5-6}\cmidrule{8-9}
Method & (5,1) & (5,5) && (5,1) & (5,5) && (5,1) & (5,5) \\
\midrule
InfoMin~\cite{infomin}    & 31.80 & 45.09 && 53.81 & 72.20 && 66.11 & 84.12 \\
CLSA~\cite{clsa}       & 34.46 & 49.16 && 50.83 & 67.93 && 67.39 & 86.15 \\
CLSA$^\dag$~\cite{clsa} & 41.09 & 58.38 && 52.87 & 70.92 && 71.57 & 88.94 \\
\midrule
R\'enyiCL    & 36.31 & 53.39 && {\bf59.73} & 82.12 && 80.68 & 94.29 \\
R\'enyiCL$^\dag$ & {\bf 42.10} & {\bf60.80} && 58.25 & {\bf82.38} && {\bf83.40} & {\bf95.71} \\
    \bottomrule
\end{tabular}
\label{tab:fewshotext}
\vspace{-10pt}
\end{table}
\begin{table}[t]
\centering
\small
\setlength\tabcolsep{1.3pt} 
\caption{Comparison of contrastive self-supervised methods with harder augmentations by downstream linear evaluation accuracy on various datasets. IN denotes ImageNet validation accuracy~(\%), and $\dag$ denotes the use of multi-crop data augmentation. {\bf Bold} entries denote the best performance.}
\begin{tabular}{@{}lcccccccccccc@{}}
\toprule
Method 	 		 & IN & CIFAR10 & CIFAR100 & Food101 & Pets  
				 & MIT67 & Flowers & Caltech101 & Cars  & Aircraft & DTD   & SUN397 \\
				 \midrule
InfoMin~\cite{infomin}  	     & 73.0     & 92.8   & 75.8   & 73.8   & 86.4 
				 & 76.9 & 91.2   & 88.5  & 49.6 & 50.5    & 75.0 & 61.2 \\
CLSA~\cite{clsa}             & 72.2     & 94.2   & 77.4    & 73.2   & 85.2 
				 & 77.1 & 90.8   & 91.2  & 48.7 & 51.6    & 75.0 & 62.0 \\
CLSA$^\dag$~\cite{clsa}      & 73.3     & 93.9   & 76.9    & 73.4   & 86.4 
				 & 77.7 & 91.2   & 90.2  & 47.4 & 50.0    & 75.5 & 62.9 \\
\midrule
R\'enyiCL   		 & 72.6     & 93.8   & 78.8 &      72.7   & 88.4 
				 & 75.8 & 94.2   & {\bf94.0}  & 62.1 & 58.8    & 74.6 & 62.1 \\
R\'enyiCL$^\dag$   & {\bf75.3}   & {\bf94.4}   & {\bf78.9}      & {\bf77.5}   & {\bf89.2}
				 & {\bf81.1} & {\bf96.2}   & {\bf94.0}  & {\bf66.4} & {\bf61.1}    & {\bf76.3} & {\bf65.9} 	\\
\bottomrule
\label{tab:linevalext}
\end{tabular}
\vspace{-20pt}
\end{table}

\subsection{CIFAR experiments}\label{apdx:cifar}
For CIFAR experiments, we use random resized cropping with a size of $32\times 32$, random horizontal flipping, color jittering, and random grayscale conversion with probability 0.2 for the base data augmentation.
For harder data augmentation, we apply RandAugment~\cite{ra}, where we use 3 operations at each iteration with a magnitude scale of 5, and RandomErasing~\cite{re}. 
We use modified ResNet-18~\cite{resnet}, where the kernel size of the first convolutional layer is converted from $7\times 7$ to $3\times3$, and we do not use Max pooling at the penultimate layer~\cite{simclr}.
We attach a 2-layer projection head of size 512-2048-128 with ReLU and BN layer attached at each FC layer, except for the last one.
We do not use a prediction head, and the temperature is set to be 0.5 for every CIFAR experiment.
We use $\gamma=1.5$ and $\alpha=1/4096$ for both CIFAR-10 and CIFAR-100 experiments.
We train for 500 epochs with the SGD optimizer, learning rate of 0.5~(without linear scaling), weight decay of 5e-4, and momentum of 0.9.
For evaluation, we train a linear classifier on the top of frozen features with 100 epochs with an SGD optimizer, learning rate of 0.3, without using weight decay. We only used random resized crop with a size of 32, and random horizontal flipping for both training and evaluation.

\begin{table}[t]
\small
\setlength\tabcolsep{1.5pt} 
\caption{Full experimental results for CovType and Higgs-100K datasets. We report top-1 accuracy~(\%) with linear evaluation. Average over 5 runs. {\bf Bold} entries denote the best performance.}
\centering
\begin{tabular}[t]{l cccc c cccc }
\toprule
        & \multicolumn{4}{c}{CovType} &&
                \multicolumn{4}{c}{Higgs-100K} \\
                \cmidrule{2-5} \cmidrule{7-10}  
    Method      & No Aug. & RM & FC & RM+FC &
                & No Aug. & RM & FC & RM+FC \\
    \midrule
    CPC         & 71.6$\pm$ 0.36 & 69.8$\pm$ 0.26 & 74.3$\pm$ 0.44 & 73.6$\pm$ 0.16 &
                & 64.7$\pm$ 0.22 & 71.3$\pm$ 0.08 & 64.9$\pm$ 0.11 & 71.4$\pm$ 0.15 \\
    MLCPC       & 71.7$\pm$ 0.26 & 70.2$\pm$ 0.15 & 74.1$\pm$ 0.39 & 74.0$\pm$ 0.45 &
                & 64.9$\pm$ 0.11 & 71.5$\pm$ 0.06 & 65.3$\pm$ 0.26 & 71.5$\pm$ 0.06 \\
    RMLCPC ($\gamma$=1.1)      
                & 72.1$\pm$ 0.31 & 70.9$\pm$ 0.42 & {\bf74.9$\pm$0.28} & 73.8$\pm$ 0.47 &
                & 65.1$\pm$ 0.29 & 71.8$\pm$ 0.17 & 65.2$\pm$ 0.13 & 71.9$\pm$ 0.06\\
    RMLCPC ($\gamma$=1.2)
                & 71.9$\pm$ 0.58 & 70.8$\pm$ 0.25 & 74.5$\pm$ 0.38 & 73.7$\pm$ 0.37 &
                & 64.5$\pm$ 0.41 & {\bf72.4$\pm$0.13} & 65.3$\pm$ 0.44 & 72.3$\pm$ 0.10\\
    \bottomrule
\end{tabular}
\label{tab:appendix_tabular}
\end{table}

\subsection{Tabular experiments}
For the data augmentation of tabular experiments, we consider simple random masking~(RM) noise that was used in~\cite{imix}, and random feature corruption~(FC) proposed in~\cite{scarf}. 
The random masking noise is analogous to RandomErasing~\cite{re} that we used in image contrastive learning.
The FC data augmentation randomly selects a subset of attributes on each tabular data and mixes it with the attributes from another data. 
Note that the random corruption is performed in the whole dataset scale~\cite{scarf}.
We refer to \cite{scarf} for detailed implementation of FC data augmentation.
We use a probability of 0.2 for both random masking noise and FC data augmentation.
For the backbone, we use a 5-layer MLP with a dimension of 2048-2048-4096-4096-8192, where all layers have a BN layer followed by RELU, except for the last one. 
We use MaxOut activation for the last layer with 4 sets, and on the top of the backbone, we attach 2-layer MLP with a dimension of 2048-2048-128 as a projection head. 
For Higgs-100K, we set the $\tau=0.1$, $\gamma=1.2$, and $\alpha=1/4096$.
For CovType, we set $\tau=0.2$, $\gamma=1.1$, and $\alpha=1/4096$.
We train for 500 epochs with an SGD optimizer of momentum 0.9, base learning rate of 0.3, and the learning rate is scheduled by half-cycle cosine annealing.
For evaluation of the CovType experiments, we train a linear classifier on the top of frozen representation for 100 epochs with a learning rate in $\{1, 3, 5\}$ and report the best validation accuracy~(\%). 
For evaluation of the Higgs-100K experiments, we use linear regression on the top of frozen representations by pseudo-inverse and report the validation accuracy~(\%).
In Table~\ref{tab:appendix_tabular}, we provide the full experimental results on CovType and Higgs-100K.
One can notice that the RM data augmentation is effective for the Higgs dataset, and FC data augmentation is effective for the CovType dataset.

\subsection{Graph experiments}
For graph contrastive learning experiments on TUDataset~\cite{tudataset}, we follow the same experimental setup in~\cite{graphcl}.
Thus, we adopt the official implementation\footnote{\url{https://github.com/Shen-Lab/GraphCL/tree/master/unsupervised_TU}}, and experiment over the hyper-parameter $\gamma\in\{1.5, 2.0, 3.0\}$ and choose the best parameter by evaluating on the validation dataset.

\section{Mutual Information Estimation}
\subsection{Implementation details}\label{appendix:midetail}
Assume we have $B$ batches of samples $\{x_i\}_{i=1}^B$ from $x_i\sim X$ and $B$ batches of samples $\{y_i\}_{i=1}^B$ from $y_i \sim Y$. Assume $(x_i,y_i)\sim P_{XY}$, i.e. positive pair, and $(x_i,y_j)\sim P_XP_Y$ for $i\neq j$. Then let $f^*(x,y)$ be optimal critic learned by maximizing CPC, MLCPC, or RMLCPC objectives. From the optimality condition, we have
\begin{align*}
    f^*(x,y) \propto \log \frac{dP_{XY}(x,y)}{\alpha dP_{XY}(x,y) + (1-\alpha)dP_X(x)dP_Y(y)} = \log\frac{r(x,y)}{\alpha r(x,y) + 1-\alpha},
\end{align*}
where $r(x,y) = \frac{dP_{XY}}{dP_XdP_Y}$ is a true density ratio. Then our idea is to recover $r$ by using $f^*$ and approximate the true mutual information. To do that, we first compute the log-normalization constant $Z$, i.e. $\frac{e^{f^*(x,y)}}{Z} = \frac{r(x,y)}{\alpha r(x,y)+1-\alpha}$. We use Monte-Carlo~\cite{owen2013monte} methods by computing empirical mean as following:
\begin{align*}
    \hat{Z} = \frac{\alpha}{B} \sum_{i=1}^B e^{f^*(x_i,y_i)} + \frac{1-\alpha}{B(B-1)}\sum_{i=1}^B\sum_{j\neq i} e^{f^*(x_i,y_j)}.
\end{align*}
Then we have following approximate density ratio $\hat{r}$:
\begin{align*}
    \hat{r}(x,y) = \frac{(1-\alpha)e^{f^*(x,y)}}{\hat{Z} - \alpha e^{f^*(x,y)}}.
\end{align*}
Since the mutual information is defined by the empirical mean of log-density ratio, we have
\begin{align*}
    \hat{\mathcal I}(X;Y) = \frac{1}{B}\sum_{i=1}^B \log \hat{r}(x_i,y_i)
\end{align*}

\subsection{MI estimation between correlated Gaussians}\label{appendix:gaussian_mi}

\paragraph{Setup.}
We follow the general procedure in~\cite{var_bounds, mlcpc}.
We sample a pair of vectors $(x,y)\in\mathbb{R}^{d\times d}$ from random variable $X\times Y$, where $X$ and $Y$ are unit normal distribution, and are correlated with correlation coefficient $\rho\in(0,1)$. 
Then the true mutual information~(MI) between $X$ and $Y$ can be computed by $\mathcal I(X;Y) =-\frac{d}{2}\log\big(1-\frac{\rho}{2}\big)$.
In our experiments, we set the initial mutual information to be $2$, and we increase the MI by $\times 2$ for every 4K iterations.
We consider joint critic, which concatenates the inputs $x,y$ and then passes through a 2-layer MLP with a dimension of $d-256-1$, and the output of the fully connected layers are followed by the ReLU activation layer. 
For optimization, we use the Adam optimizer with learning rate of $10^{-3}$, $\beta_1=0.9$, $\beta_2=0.999$, batch size of 128.

\paragraph{Results.}
In Figure~, we depict the curve of training curves and MI estimation that we proposed in Section 3.4. 
We consider $\alpha$-CPC, $\alpha$-MLCPC, $(\alpha,\gamma)$-RMLCPC with $\gamma=2.0$. 
Remark that when $\alpha=0.0$, the MI estimators exhibit large variance as we proved in Theorem 3.1. 
But for large enough $\alpha > 0$, the training objectives tend to have low variance, and thus we have a low variance MI estimator.

From Theorem 4.1, since there are total $128 \times 127$ negative pairs, we need to select $\alpha \propto 1/128$. Therefore, we divide the $\alpha$ by $128$ for each experiment in Figure~\ref{fig:gaussian_mi}. 
One can observe that even though the optimal $\alpha$ varies across the objectives, choosing $\alpha$ around 1/128 gives a fairly good estimation of MI. 

Note that the skew divergence is not only applicable to DV objectives, but also for the NWJ~\cite{nwj} objective, which is another variational lower bound of KL divergence defined as following:
\begin{align*}
    \mathcal{I}_{\tt{NWJ}}(f) \coloneqq \mathbb{E}_{P_{XY}}[f(x,y)] - \mathbb{E}_{P_XP_Y}[e^{f(x,y)-1}],
\end{align*}
and we have $\mathcal I(X;Y) = D_{\tt{KL}}(P_{XY}\| P_X P_Y) = \sup_f \mathcal{I}_{\tt{NWJ}}(f)$~\cite{nwj}.
Note that the NWJ objective is proven to have large variance as the true MI is too large~\cite{song2019understanding}. However, we show that by introducing the NWJ objective for $\alpha$-skew KL divergence, one can lower the variance of the training objective, and perform MI estimation as we explained in Section 3.4.
We refer the NWJ variational lower bound of $\alpha$-skew KL divergence as $\alpha$-NWJ, and is defined by following:
\begin{align*}
    \mathcal{I}_{\tt{NWJ}}^{(\alpha)}(f) \coloneqq \mathbb{E}_{P_{XY}}[f(x,y)] - \alpha\mathbb{E}_{P_{XY}}[e^{f(x,y)-1}] - (1-\alpha)\mathbb{E}_{P_XP_Y}[e^{f(x,y)-1}],
\end{align*}
and we have $ D_{\tt{KL}}^{(\alpha)}(P_{XY}\| P_X P_Y) = \sup_f \mathcal{I}_{\tt{NWJ}}^{(\alpha)}(f)$. Lastly, in Figure~, we show that one can perform stable MI estimation through $\alpha$-NWJ objective by choosing appropriate $\alpha$.

\subsection{ Mutual Information between the views and InfoMin principle.}
\paragraph{Implementation.}
We present the pseudocode for estimation of MI between the views in Algorithm~\ref{alg:code2}. 
In Figure~, we compute the MI for every iteration, and report the mean value of MI for the last epoch of training.
\begin{algorithm}[t]
\caption{Pseudocode for MI estimation between the views}
\label{alg:code2}
\definecolor{codeblue}{rgb}{0.25,0.5,0.5}
\definecolor{codekw}{rgb}{0.85, 0.18, 0.50}
\begin{lstlisting}[language=python]
# aug: data augmentation that generates views
# x: batch of samples
# f: critic that takes a pair of inputs
# extract_pos_and_neg: a function that extracts positive and negative pairs (Alg. 1)

def MI_estimation(x, f, aug, alpha, gamma):
    x1, x2 = aug(x), aug(x)  # two views
    pos, neg = extract_pos_and_neg(f(x1, x2)) # extract positive and negative pairs
    Z = alpha*pos.exp() + (1-alpha)*neg.exp().mean(dim=1)  # normalization constant
    r_alpha = pos.exp() / Z
    r_true = (1-alpha)*r_alpha / (1-alpha*r_alpha)
    return r_true.log().mean()
\end{lstlisting}
\end{algorithm}

\paragraph{InfoMin principle and R\'enyiCL.}
\begin{figure}[t]
    \centering
    \includegraphics[width=8cm]{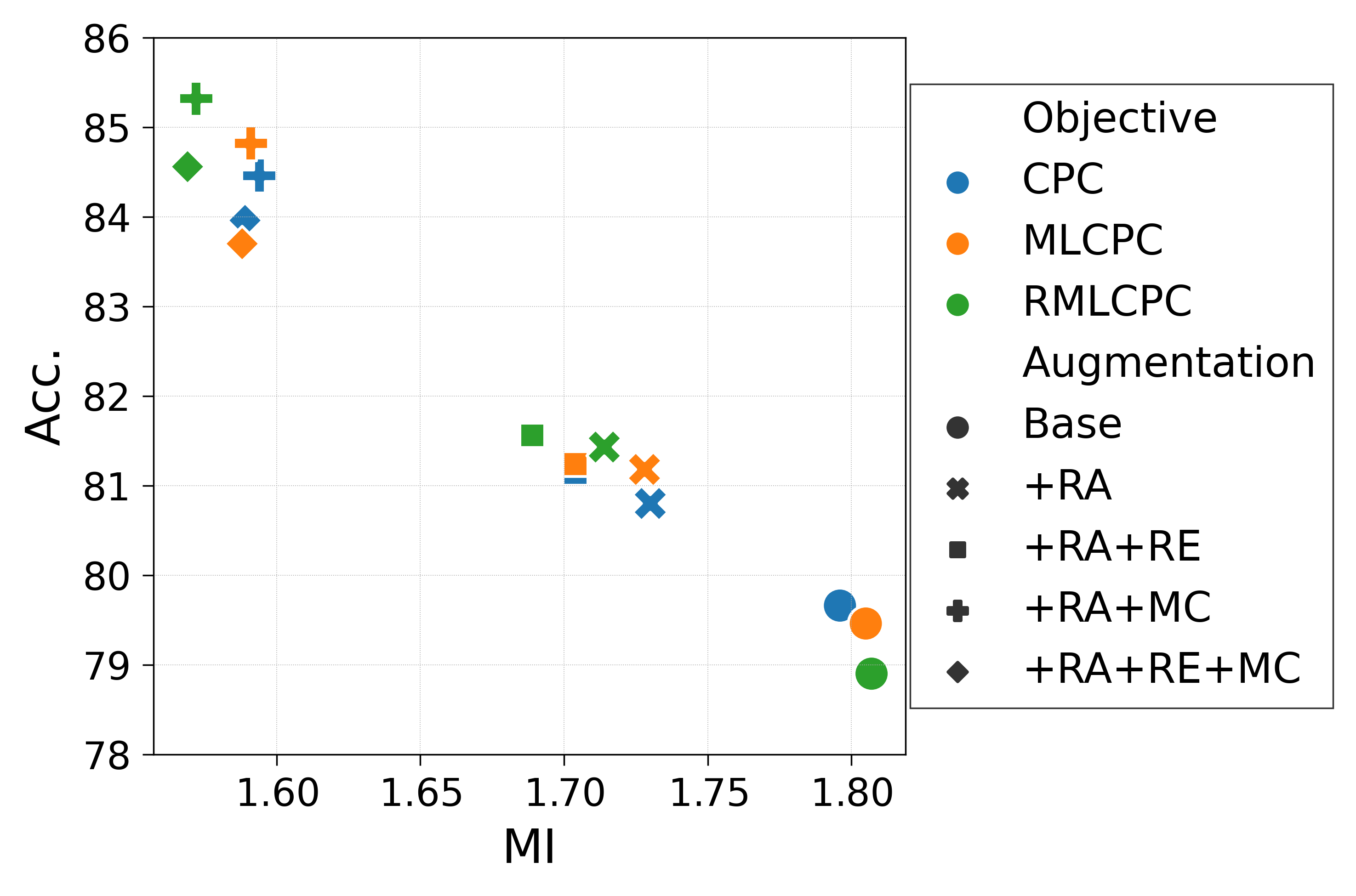}
    \caption{MI between the views and validation accuracy on ImageNet-100. RA: RandAugment~\cite{ra}, RE: RandomErasing~\cite{re}, MC: Multi-crops~\cite{swav}.}
    \label{fig:mi}
\end{figure}

Following \cite{infomin}, we plot the mutual information versus linear evaluation accuracy~(\%) by training with different data augmentations and contrastive objectives.
In Figure~\ref{fig:mi}, we observe that for harder data augmentations, the RMLCPC induces lower MI than that of CPC and MLCPC, while showing better downstream performances.
This is because of the effect of easy positive and hard negative sampling explained in Section~\ref{sec:rmlcpc}.
Note that InfoMin~\cite{infomin} principle argues that the optimal views must share sufficient and minimal information about the downstream tasks. 
Here, one can observe that R\'enyiCL intrinsically follows the InfoMin principle as it extracts minimal and sufficient information between the views. 
Therefore, R\'enyiCL intrinsically follows InfoMin principle~\cite{infomin} that it extracts minimal and sufficient information between the views.
Here, we establish the orthogonal relationship between the InfoMin principle~\cite{infomin} and R\'enyiCL: InfoMin principle aims to find views that share minimal and sufficient information, while R\'enyiCL can extract minimal and sufficient information from given views.

\begin{figure}
    \small
    \centering
        \begin{tabular}{c}
            \includegraphics[width=\textwidth]{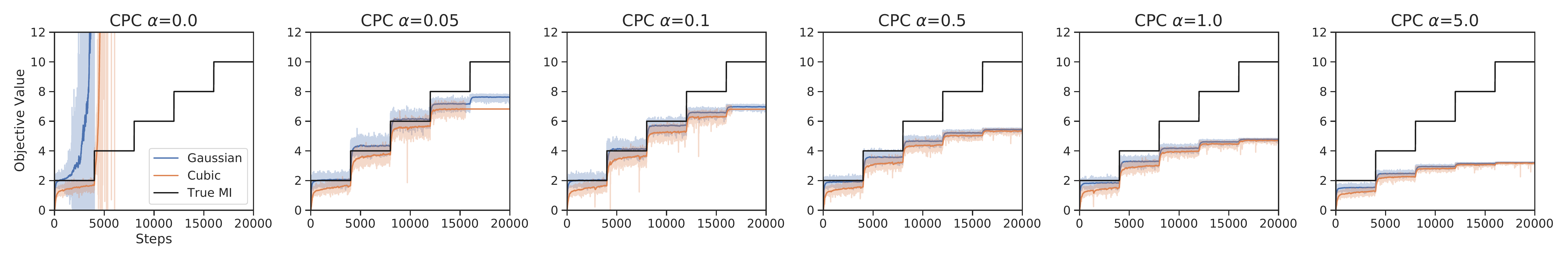} \\
            \includegraphics[width=\textwidth]{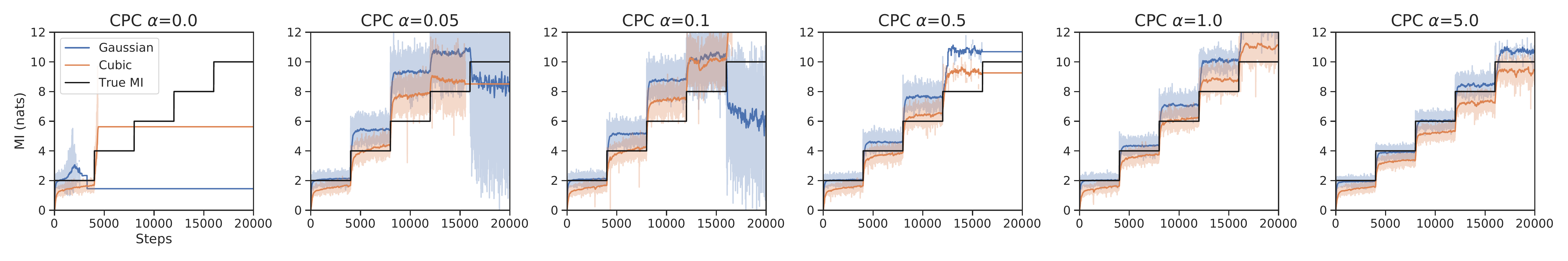} \\
            (a) Top: $\alpha$-CPC objective, Bottom: MI estimation from $\alpha$-CPC \\
            \\
            \includegraphics[width=\textwidth]{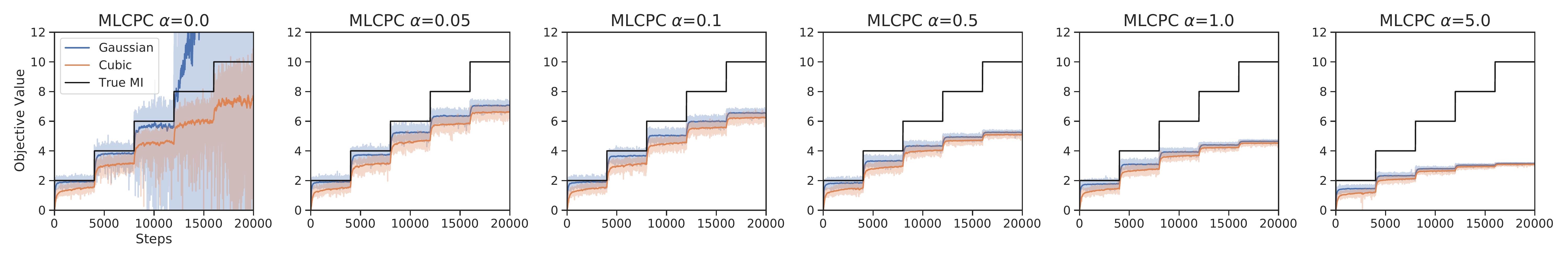} \\
            \includegraphics[width=\textwidth]{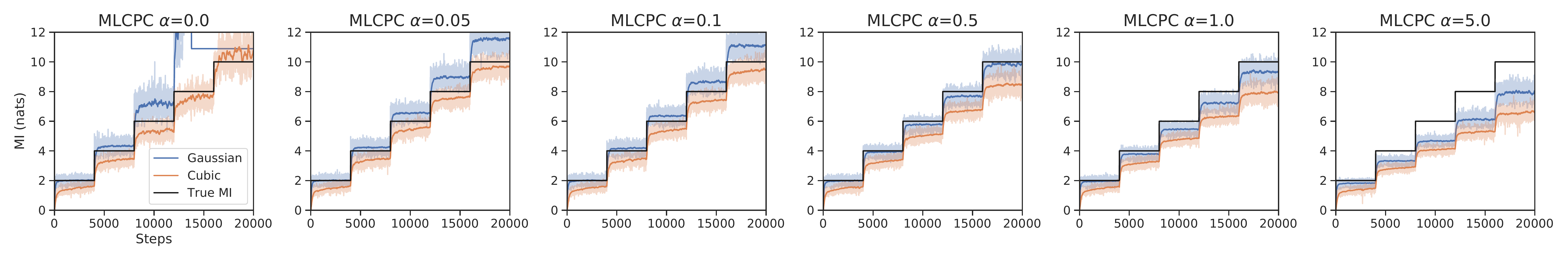} \\
            (b) Top: $\alpha$-MLCPC objective, Bottom: MI estimation from $\alpha$-MLCPC \\
            \\
            \includegraphics[width=\textwidth]{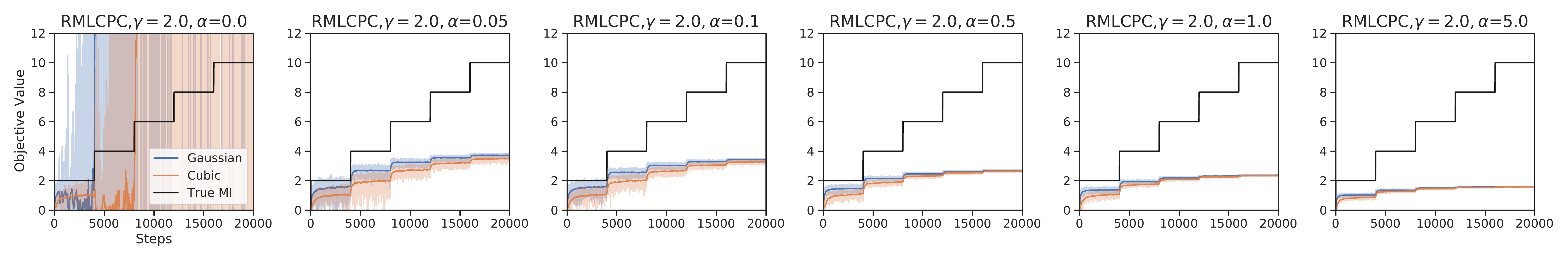} \\
            \includegraphics[width=\textwidth]{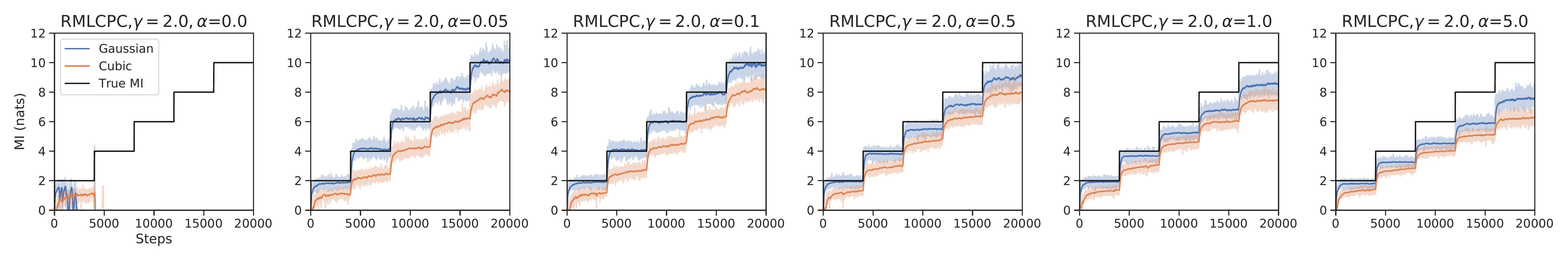} \\
            (c) Top: $(\alpha,\gamma)$-RMLCPC objective, Bottom: MI estimation from
            $(\alpha,\gamma)$-RMLCPC\\
            \\
            \includegraphics[width=\textwidth]{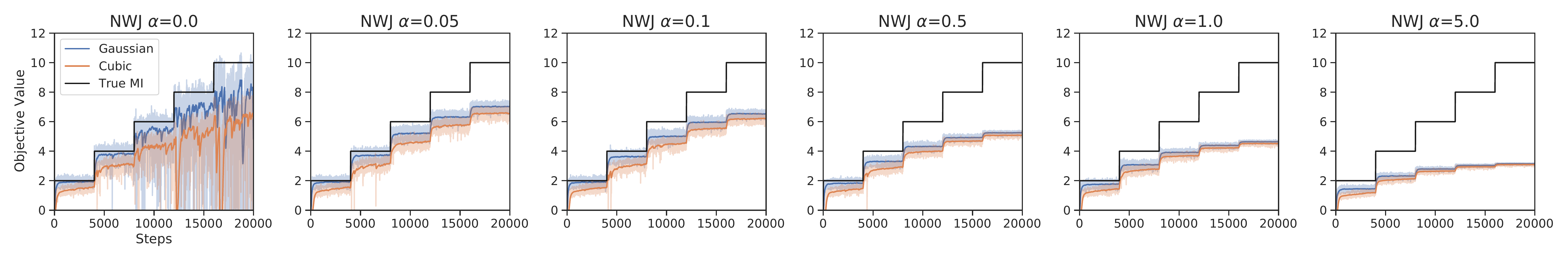} \\
            \includegraphics[width=\textwidth]{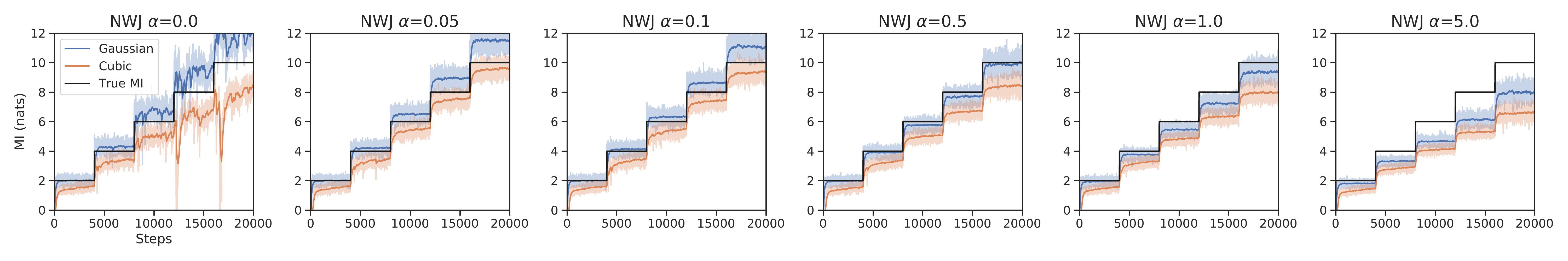} \\
            (d) Top: $\alpha$-NWJ, Bottom: MI estimation from $\alpha$-NWJ\\
        \end{tabular}
    \caption{The plots for Gaussian MI estimation. We plot the values of training objectives, i.e., the negative of each objective, and the estimated MI value. For RMLCPC objective, we use $\gamma=2.0$. Note that we divide the $\alpha$ by the batch size of $128$.}
    \label{fig:gaussian_mi}
\end{figure}

\section{Hyperparameter Sensitivity Analysis}
The temperature is a hyperparameter that greatly affects the performance of contrastive learning~\cite{simclr, supcon}. 
Here, we show that RMLCPC objective is robust to the choice of temperature hyperparameter. 
We conduct simple ablation studies on CIFAR-10 and CIFAR-100 with temperature $\tau= 0.5, 1.0$ and $\gamma=1.5, 2.0$.
We use the same experimental setup described in \ref{apdx:cifar} except the value of temperature. 
For comparison, we use same data augmentation setup; base data augmentation and hard augmentation using RandAugment and RandomErasing.
In Table~\ref{tab:hype}, we report the linear evaluation accuracy of each setting trained on CIFAR-10 and CIFAR-100 datasets.
When using large temperature, i.e., $\tau=1.0$, CPC and MLCPC degrades their performance when using harder data augmentation.
On the other hand, RMLCPC consistently achieves better performance when using harder augmentations on both $\tau=0.5$ and $\tau=1.0$.
Especially for CIFAR-100 dataset, using RMLCPC objective with $\tau=1.0$ shows similar performance with using MLCPC objective with $\tau=0.5$.
Thus, our R\'enyiCL does not only provide empirical gain, but also it could be a good starting point to search for effective hyperparameter due to their robustness to temperature. 

\begin{table}[ht]
\small
\setlength\tabcolsep{5pt} 
\caption{Sensitivity of CPC, MLCPC, RMLCPC on temperature. We report the linear evaluation accuracy~(\%) on CIFAR-10 and CIFAR-100 datasets. {\bf Bold} entries denote the best performance. }
\centering
\begin{tabular}[t]{l cc c cc c cc c cc }
\toprule
        & \multicolumn{5}{c}{CIFAR-10} &&
                \multicolumn{5}{c}{CIFAR-100} \\
                \cmidrule{2-6} \cmidrule{8-12}
                & \multicolumn{2}{c}{$\tau=1.0$} && \multicolumn{2}{c}{$\tau=0.5$} &
                & \multicolumn{2}{c}{$\tau=1.0$} && \multicolumn{2}{c}{$\tau=0.5$} \\
                \cmidrule{2-3} \cmidrule{5-6} \cmidrule{8-9} \cmidrule{11-12}
    Method      & Base & Hard && Base & Hard &
                & Base & Hard && Base & Hard \\
    \midrule
CPC 			& 91.7 & 90.9 && 91.7 & 91.9 &
				& 63.0 & 62.2 && 65.4 & 67.1 \\
MLCPC 			& 91.6 & 90.8 && 91.9 & 92.1 &
				& 62.8 & 61.9 && 65.6 & 66.6 \\
				\midrule
RMLCPC~$(\gamma=1.5)$ 		& 91.6 & 91.7 && 91.4 & {\bf92.5} &
				& 64.2 & 64.2 && 65.4 & 68.0 \\
RMLCPC~$(\gamma=2.0)$ 		& 91.3 & 92.2 && 90.7 & {\bf92.5} &
				& 64.5 & 66.2 && 64.5 & {\bf68.5} \\
    \bottomrule
\end{tabular}
\label{tab:hype}
\end{table}

%


\end{document}